\documentclass[journal,twocolumn]{IEEEtran}
\usepackage{graphicx} 
\usepackage{bbm}
\usepackage{amsmath,amssymb,amsthm} 
\usepackage{fontenc}
\usepackage{acronym}
\usepackage{cite}
\usepackage{url}
\usepackage{booktabs}
\usepackage{xcolor}

\makeatletter%
\if@twocolumn%
\else% \@twocolumnfalse
\fi%
\makeatother%
\acrodef{SDP}[SDP]{semidefinite programming}

\newtheorem{Theorem}{Theorem}
\newtheorem{Lemma}{Lemma}
\newtheorem{Remark}{Remark}

\title{Semidefinite Programming for Community Detection with Side Information}

\author{ Mohammad~Esmaeili,
Hussein~Metwaly~Saad,~\IEEEmembership{Member,~IEEE}, and~Aria~Nosratinia,~\IEEEmembership{Fellow,~IEEE}
\thanks{M. Esmaeili and A. Nosratinia are with the Department of Electrical and Computer Engineering, The University of Texas at Dallas, Richardson, TX 75083-0688, USA, Email: esmaeili@utdallas.edu, aria@utdallas.edu. H. Metwaly Saad is with the Department of Electrical and Computer Engineering, Virginia Tech, Blacksburg, VA 24061, USA. Email: husseinm19@vt.edu}
\thanks{This work was supported in part by the grant 2008684 from the National Science Foundation.}
}

\def\Zsdp{\widehat{Z}}

\begin{document}

\maketitle

\begin{abstract}
This paper produces an efficient \ac{SDP} solution for community detection that incorporates non-graph data, which in this context is known as side information.
\ac{SDP} is an efficient solution for standard community detection on graphs. 
We formulate a semi-definite relaxation for the maximum likelihood estimation of node labels, subject to observing {\em both} graph and non-graph data. This  formulation is distinct from the \ac{SDP} solution of standard community detection, but maintains its desirable properties. We calculate the exact recovery threshold for three types of non-graph information, which in this paper are called side information: partially revealed labels, noisy labels, as well as multiple observations (features) per node with arbitrary but finite cardinality. We find that \ac{SDP} has the same exact recovery threshold in the presence of side information as maximum likelihood with side information. Thus, the methods developed herein are computationally efficient as well as asymptotically accurate for the solution of community detection in the presence of side information. Simulations show that the asymptotic results of this paper can also shed light on the performance of \ac{SDP} for graphs of modest size.
\end{abstract}


%%%%%%%%%%%%%%%%%%%%%%%%%%%%%%%%%%%%%%%%%%%%%%%%%%%%%%%%%%%%%%%%%%%%%%%%%%%%%%%%%%%%%%%%%%%%%%%%%%%%%%%%%%%%%%%%%%%%%%%%%%%
\section{Introduction} 
\label{Introduction}
Detecting communities (or clusters) in graphs is a fundamental problem that has many applications, such as finding like-minded people in social networks~\cite{Ref5}, and improving recommendation systems~\cite{Ref6}.
Community detection is affiliated with various problems in network science such as network structure reconstruction~\cite{Keke-Huang-2019}, networks with dynamic interactions~\cite{Zhen-Wang-2020}, and complex networks~\cite{Keke-Huang-2020}. 
Random graph models~\cite{CommunityDetectionInGraphs2010, esmaeili2021community} are frequently used in the analysis of community detection, prominent examples of which include the stochastic block model~\cite{esmaeili2021community,Ref1,Ref4} and the censored block model~\cite{saade2015spectral, hajek2015exact}. 
In the context of these models, community detection recovers latent node labels (communities) by observing the edges of a graph.

Community detection utilizes several metrics for residual error as the size of the graph grows: correlated recovery \cite{Ref8, Ref9, Ref10, Ref11} (recovering the hidden community better than random guessing), weak recovery \cite{Ref12, Ref13, Ref15} (the fraction of misclassified labels in the graph vanishes with probability converging to one), and exact recovery~\cite{Ref4, Ref16, Ref17} (all nodes are classified correctly with high probability).  Recovery techniques include spectral methods~\cite{Ref4,Statistical.computational.Tradeoffs.in.Planted.Problems.and.Submatrix.Localization.with.a.Growing.Number.of.Clusters.and.Submatrices}, belief propagation~\cite{mossel2014belief}, and \ac{SDP} relaxation~\cite{amini2018semidefinite}.

Semidefinite programming is a computationally efficient convex optimization technique that has shown its utility in solving signal processing problems~\cite{SDP.SignalProcessing, ConvexOptimization.SignalProcessing}.
In the context of community detection, \ac{SDP} was introduced in~\cite{Ref24}, where it was used for solving a minimum bisection problem, obtaining a sufficient condition that is not optimal. In~\cite{Ref25}, a \ac{SDP} relaxation was considered for a maximum bisection problem. For the binary symmetric stochastic block model,~\cite{Ref18} showed that the \ac{SDP} relaxation of maximum likelihood can achieve the optimal exact recovery threshold with high probability. These results were later extended to more general models in~\cite{Ref19}.
Also, \cite{esmaeili2020community} showed the power of \ac{SDP} for solving a community detection problem in graphs with a secondary latent variable for each node. 

Community detection on graphs has been widely studied in part because the graph structure is amenable to analysis and admits efficient algorithms. In practice, however, the
available information for inference is often not purely graphical. For instance, in a citation network, beside the names of authors, there are some additional {\em non-graph} information such as keywords and abstract that can be used and improve the performance of community detection algorithms.
For illustration, consider public-domain libraries such as Citeseer and Pubmed. Citation networks in these libraries have been the subject of several community detection studies, which can be augmented by incorporating individual (non-graph) attributes of the documents that affect the likelihood of community memberships.
%Real-world citation networks such as Citeseer, Cora, and Pubmed are examples of graphs with side information. Each of these networks are represented by an adjacency matrix and a feature matrix containing side information. In these networks, articles are considered as nodes, the article citations determine the graph edges, and a sparse bag-of-words vector (extracted from the title and the abstract of each article) is used as a vector of side information for each node.
% Citeseer, Cora, and Pubmed are three examples of real-world graph datasets with non-graph  information that is provided by feature matrices. For each dataset, the feature matrix is independent of the adjacency matrix of the graph and provides some additional information about the nodes. 
% In these networks, articles are considered as nodes. The article citations determine the edges connected to the corresponding node. 
% Also, a sparse bag-of-words vector, extracted from the title and the abstract of each article, is used as a vector of features for that node. 

The non-graph data assisting in the solution of graph problems is called {\em side information.} 
%The notion of side information has been introduced and studied in the literature. 
In~\cite{Ref7,our2}, the effect of side information on the phase transition of the exact recovery was studied for the binary symmetric stochastic block model. In~\cite{our3,ISIT-2018-1,ISIT-2018-2}, the effect of side information was studied on the phase transition of the weak and exact recovery as well as the phase transition of belief propagation in the single community stochastic block model. %In~\cite{Ref20}, the behavior of belief propagation in the presence of noisy label information in the binary stochastic block model is studied. 
%The effect of knowing a growing fraction of labels for weak recovery is investigated in~\cite{Ref21}. 
%However, to the best of our knowledge, the problem of solving graph inference problems with non-graphical side information has not been formulated or analyzed via convex optimization.
The impact of side information on the performance of belief propagation was further studied in~\cite{Ref20,ISIT-2018-1}. 

The contribution of this paper is the analysis \color{black} of  the impact of side information on \ac{SDP} solutions for community detection. More specifically, we study the behavior of the \ac{SDP} detection threshold under the exact recovery metric. We consider graphs following the binary censored block model and the binary symmetric stochastic block model. 
%We study the impact of partially revealed label side information and noisy label side information. 
We begin with the development of \ac{SDP} for partially-revealed labels and noisy labels, which are easier to grasp and visualize. This builds intuition for the more general setting, in which we study side information with multiple features per node, each of which is a random variable with arbitrary but finite cardinality. The former results also facilitate the understanding and interpretation of the latter. 
%Only the novel parts of each proof are presented, and any overlapping developments are omitted in the interest of brevity.
%
%
%To recap, in the problem of community detection in the presence of non-graph data, the main contribution of the present work is to comprehensively formulate and analyze efficient solutions via semi-definite programming. 
Most categories of side information give rise to a complete quadratic form in the likelihood function, which presents challenges in the analysis of their semidefinite programming relaxation. Overcoming these challenges is one of the main technical contributions of the present work.

Simulation results show that the thresholds calculated in this paper can also shed light on the understanding of the behavior of \ac{SDP} in graphs of modest size.

Notation: 
Matrices and vectors are denoted by capital letters, and their elements with small letters. $\mathbf{I}$ is the identity matrix and $\mathbf{J}$ the all-one matrix. %$\mathcal{S}^{n}$ is the set of all $n \times n$ symmetric matrices. 
$S \succeq 0$ indicates a positive semidefinite matrix and  $S \ge 0$ a matrix with non-negative entries. $||S||$ is the spectral norm, $\lambda_{2}(S)$ the second smallest eigenvalue (for a symmetric matrix), and $\langle \cdot,\cdot\rangle$ is the inner product. We abbreviate $[n] \triangleq \{ 1, \cdots,n \}$. Probabilities are denoted by $\mathbb{P}(\cdot)$ and random variables with Bernoulli and Binomial distribution are indicated by $\mathrm{Bern}(p)$ and $\mathrm{Binom}(n,p)$, respectively.

%%%%%%%%%%%%%%%%%%%%%%%%%%%%%%%%%%%%%%%%%%%%%%%%%%%%%%%%%%%%%%%%%%%%%%%%%%%%%%%%%%%%%%%%%%%%%%%%%%%%%%%%%%%%%%%%%%%%%%%%%%%
\section{System Model}

This paper analyzes community detection in the presence of a graph observation as well as individual node attributes. The graphs in this paper follow the binary stochastic block model and the censored block model, and side information is in the form of either partially revealed labels, noisy labels, or an alphabet other than the labels.

This paper considers a fully connected regime, guaranteeing that exact recovery is possible.
Throughout this paper, the graph adjacency matrix is denoted by $G$. Node labels are independent and identically distributed across $n$, with labels $+1$ and $-1$. The vector of node labels is denoted by $X$, and a corresponding vector of side information is denoted by $Y$.
The log-likelihood of the graph and side information is
\begin{equation*}
\log \mathbb {P}(G,Y|X) = \log\mathbb{P}(G|X)+ \log\mathbb {P}(Y|X) ,
\end{equation*}
i.e., $G$ and $Y$ are independent given $X$. 

%\subsection{Generative Models for Graphical Observations}
\subsection{Binary Censored Block Model}
The model consists of an Erd\H{o}s-R\'enyi graph with $n$ nodes and edge probability $p=a\frac{\log n}{n}$ for a fixed $a>0$. The nodes belong to two communities represented by the binary node labels, which are latent. The entries $G_{ij}\in \{-1,0,1\}$ of the weighted adjacency matrix of the graph have a distribution that depends on the community labels $x_i$ and $x_j$ as follows:
\[
G_{ij} \sim \begin{cases}
p(1-\xi) \delta_{+1} + p\xi \delta_{-1} +(1-p)\delta_{0} & \text{ when } x_i = x_j\\
p(1-\xi) \delta_{-1} + p\xi \delta_{+1} +(1-p)\delta_{0} & \text{ when } x_i \neq x_j
\end{cases}
\]
%Then the edges of the graph is weighted according to the community structure in the following manner. 
%Let $G$ denote the weighted adjacency matrix of the binary censored block model with entries $G_{ij}$. When nodes $i$ and $j$ are in the same community, $G_{ij} \sim p(1-\xi) \delta_{+1} + p\xi \delta_{-1} +(1-p)\delta_{0}$; otherwise $G_{ij} \sim p(1-\xi) \delta_{-1} + p\xi \delta_{+1} +(1-p)\delta_{0}$, 
where $\delta$ is Dirac delta function and $\xi \in [0, \frac{1}{2}]$ is a constant. Further, $G_{ii}=0$ and $G_{ij}=G_{ji}$. For all $j>i$, the edges $G_{ij}$ are mutually independent conditioned on the node labels.  The log-likelihood of $G$ is
\begin{equation}
\label{BCBM-equ1}
\log \mathbb{P}(G|X) = \frac{1}{4}T_{1}X^{T}GX+C_{1}, 
\end{equation}
where $T_{1} \triangleq \log \Big( \frac{1-\xi}{\xi} \Big)$ and $C_{1}$ is a deterministic scalar.

\subsection{Binary Symmetric Stochastic Block Model}
%Let $G$ denote the adjacency matrix of the binary symmetric stochastic block model. 
In this model, if nodes $i,j$ belong to the same community, 
%$G_{ij} \sim p \delta_{+1} +(1-p)\delta_{0}$ ]
$G_{i,j}\sim \mathrm{Bern}(p)$, otherwise $G_{ij} \sim \mathrm{Bern}(q)$ with
\begin{equation*}
    p=a\frac{\log n}{n},\qquad q=b\frac{\log n}{n},
\end{equation*}
and $a\geq b>0$. 
Then the log-likelihood of $G$ is 
\begin{equation}
\label{BSSBM-equ1}
\log \mathbb{P}(G|X) = \frac{1}{4}T_{1}X^{T}GX+C_{2} ,
\end{equation}
where $T_{1}\triangleq \log \Big (\frac{p(1-q)}{q(1-p)} \Big )$ and $C_{2}$ is a deterministic scalar.

\subsection{Side Information: Partially Revealed Labels}

Partially-revealed side information vector $Y$ consists of elements that with probability $1-\epsilon$ are  equal to the true label and with probability $\epsilon$ take value $0$, i.e., are erased.

Conditioned on each node label, the corresponding side information is assumed independent from other labels and from the graph edges. Thus, the log-likelihood of $Y$ is
\begin{equation}
\label{P-equ1}
\log \mathbb{P}(Y|X)= Y^{T}Y \log \bigg( \frac{1-\epsilon}{\epsilon} \bigg)+n \log(\epsilon).
\end{equation}

\subsection{Side Information: Noisy Labels }

Noisy-label side information vector $Y$ consists of elements that with probability $1-\alpha$ agree with the true label ($y_i=x^*_i$) and with probability $\alpha$ are erroneous ($y_i=-x^*_i$), where $\alpha \in (0, 0.5)$. Then the log-likelihood of $Y$ is
\begin{equation}
\label{N-equ1}
\log \mathbb{P}(Y|X)=\frac{1}{2}T_{2}X^{T}Y+T_{2}\frac{n}{2}+n \log \alpha  ,
\end{equation}
where $T_{2} \triangleq \log \big( \frac{1-\alpha}{\alpha} \big)$.

\subsection{Side Information: Multiple Variables \& Larger Alphabets}

In this model, we disengage the cardinality of side information alphabet from the node latent variable, and also allow for more side information random variables per node. This is motivated by practical conditions where the available non-graph information may be different from the node latent variable, and there may be multiple types of side information with varying utility for the inference.

Formally, $y_{i,k}$ is the random variable representing feature $k$ at node $i$. Each feature has cardinality $M_k$ that is finite and fixed across the graph. We group these variables into a vector $y_i$ of dimension $K$, representing side information for node $i$, and group the vectors into a matrix $Y$ representing all side information for the graph.\footnote{If vectors $y_i$ have unequal dimension, matrix $Y$ will accommodate the largest vector, producing vacant entries that are defaulted to zero.} 

Without loss of generality, the alphabet of each feature $k$ is the set of integers $\{1, \ldots,M_k\}$. The posterior probability of the features are denoted by
\begin{align*}
\alpha_{+,m_{k}}^{k} &\triangleq \mathbb{P} ( y_{i,k}=m_{k} | x_{i}=1 ),\\ 
\alpha_{-,m_{k}}^{k} &\triangleq \mathbb{P}( y_{i,k}=m_{k} | x_{i}=-1  ) ,
\end{align*}
where $m_{k}$ indexes the alphabet of feature $k$.
Then the log-likelihood of $Y$ is
\begin{align*}
\log \mathbb{P} & ( Y|X )= \sum_{i=1}^{n} \log \mathbb{P} ( y_{i}|x_{i} ) \\
=&\frac{1}{2} \sum_{i=1}^{n} x_{i} \sum_{k=1}^{K} \sum_{m_{k}=1}^{M_{k}} \mathbbm{1}_{y_{i,k} = m_{k} } \log\bigg( \frac{\alpha_{+,m{k}}^{k}}{\alpha_{-,m_{k}}^{k}}\bigg) \\
&+\frac{1}{2} \sum_{i=1}^{n} \sum_{k=1}^{K} \sum_{m_{k}=1}^{M_{k}} \mathbbm{1}_{y_{i,k} = m_{k} } \log (\alpha_{+,m_{k}}^{k} \alpha_{-,m_{k}}^{k} ) ,
\end{align*}
where $\mathbbm{1}$ is the indicator function. 
Define 
\[
\tilde{y}_{i} \triangleq \sum_{k=1}^{K} \sum_{m_{k}=1}^{M_{k}} \mathbbm{1}_{\{y_{i,k} = m_{k}\} } \log \bigg(\frac{\alpha_{+,m_{k}}^{k}}{\alpha_{-,m_{k}}^{k}}\bigg),
\]
and $\tilde{Y}\triangleq [\tilde{y}_{1}, \tilde{y}_{2}, \ldots , \tilde{y}_{n}]^{T}$.
Then the log-likelihood of $Y$ is
\begin{equation}
\label{G-equ1}
\log \mathbb{P} ( Y|X  )= \frac{1}{2} X^{T}\tilde{Y} + C_3 , 
\end{equation}
for some constant $C_3$. In the remainder of this paper, side information thus defined is referred to as {\em general side information}.
%%%%%%%%%%%%%%%%%%%%%%%%%%%%%%%%%%%%%%%%%%%%%%%%%%%%%%%%%%%%%%%%%%%%%%%%%%%%%%%%%%%%%%%%%%%%%%%%%%%%%%%%%%%%%%%%%%%%%%%%%%%
\section{Detection via \ac{SDP}}

For organizational convenience, the main results of the paper are concentrated in this section.

For the formulation of \ac{SDP}, we utilize the additional variables $Z\triangleq XX^{T}$ and $W \triangleq YY^{T}$. Also, let $Z^{*} \triangleq X^{*}X^{*T}$.

\subsection{Censored Block Model with Partially Revealed Labels}
Combining~\eqref{BCBM-equ1} and~\eqref{P-equ1}, the maximum likelihood detector is
\begin{align}
\label{BCBM-P-equ1}
\hat{X} =&  \underset{X}{\arg\max} ~X^{T}GX  \nonumber\\
&\text{subject to} \quad x_{i} \in \{\pm  1 \},\quad i\in[n] \nonumber\\
&\quad \quad \quad \quad \quad X^{T}Y=Y^{T}Y,
\end{align}
where the constraint $X^{T}Y=Y^{T}Y$ ensures that detected values agree with available side information.
This is a non-convex problem, therefore we consider a convex relaxation~\cite{Ref16,Ref26}. 
Replacing $x_{i} \in \{\pm  1 \}$ with $Z_{ii}=1$, and $X^{T}Y= \pm Y^{T}Y$ with $\langle Z,W\rangle =(Y^{T}Y)^2$,
\begin{align}
\label{BCBM-P-equ1-1}
\widehat{Z}=&\underset{Z}{\arg\max} ~\langle Z,G\rangle \nonumber\\
&\text{subject to} \quad Z=XX^{T} \nonumber\\
& \quad \quad \quad \quad \quad Z_{ii}=1,\quad i\in [n] \nonumber\\
& \quad \quad \quad \quad \quad \langle Z,W\rangle =(Y^{T}Y)^2.
\end{align}
By relaxing the rank-one constraint introduced via $Z$, we obtain the following \ac{SDP} relaxation:
\begin{align}
\label{BCBM-P-equ2}
\Zsdp=&\underset{Z}{\arg\max} ~\langle Z,G\rangle \nonumber\\
&\text{subject to} \quad Z\succeq 0 \nonumber\\
& \quad \quad \quad \quad \quad Z_{ii}=1,\quad i\in [n] \nonumber\\
& \quad \quad \quad \quad \quad \langle Z,W\rangle =(Y^{T}Y)^2.
\end{align}
Let $\beta \triangleq \lim_{n \rightarrow \infty} -\frac{\log \epsilon}{\log n}$, where $\beta \geq 0$.
\begin{Theorem} 
\label{Theorem 1}
Under the binary censored block model and partially revealed labels, if
\begin{equation*}
    a(\sqrt{1-\xi}-\sqrt{\xi})^2+\beta>1 ,
\end{equation*}
then the \ac{SDP} estimator is asymptotically optimal, i.e., $\mathbb{P}(\Zsdp=Z^{*})\geq 1-o(1)$.
\end{Theorem}
\begin{proof}
See Appendix~\ref{Proof-Theorem-1}.
\end{proof}

\begin{Theorem} 
\label{Theorem 2}
Under the binary censored block model and partially revealed labels, if
\begin{equation*}
    a(\sqrt{1-\xi}-\sqrt{\xi})^2+\beta<1 , 
\end{equation*}
then for any sequence of estimators $\widehat{Z}_{n}$, $\mathbb{P}(\widehat{Z}_{n}=Z^{*}) \rightarrow 0$ as $n \rightarrow \infty $.
\end{Theorem}
\begin{proof}
See Appendix~\ref{Proof-Theorem-2}.
\end{proof}

%%%%%%%%%%%%%%%%%%%%%%%%%%%%%%%%%%%%%%%%%%%
\subsection{Censored Block Model with Noisy Labels}
Combining~\eqref{BCBM-equ1} and~\eqref{N-equ1}, the maximum likelihood detector is 
\begin{align}
\label{BCBM-N-equ1}
\hat{X} = & \underset{X}{\arg\max}~ T_{1}X^{T}GX+2T_{2}X^{T}Y  \nonumber\\
&\text{subject to} \quad x_{i} \in \{\pm  1 \},\quad i\in[n].
\end{align}
Then~\eqref{BCBM-N-equ1} is equivalent to
\begin{align}
\label{BCBM-N-equ2}
\widehat{Z}=&\underset{Z,X}{\arg\max} ~T_{1}   \langle G, Z   \rangle + 2T_{2} X^{T}Y  \nonumber\\
&\text{subject to} \quad Z = XX^{T} \nonumber\\
& \quad \quad \quad \quad \quad Z_{ii} =1, \quad i \in  [n ]. 
\end{align}
Relaxing the rank-one constraint, using
\begin{equation*}
Z-XX^{T}\succeq 0  \Leftrightarrow \begin{bmatrix} 1 & X^{T}\\  X & Z \end{bmatrix} \succeq 0 , 
\end{equation*}
yields the \ac{SDP} relaxation of~\eqref{BCBM-N-equ2}:
\begin{align}
\label{BCBM-N-equ3}
\Zsdp=&\underset{Z,X}{\arg\max} ~T_{1}   \langle G, Z   \rangle + 2T_{2} X^{T}Y  \nonumber\\
&\text{subject to} \quad \begin{bmatrix} 1 & X^{T}\\  X & Z \end{bmatrix} \succeq 0 \nonumber\\
& \quad \quad \quad \quad \quad Z_{ii}=1,\quad i\in [n] .
\end{align}
% Define  $\mathcal{Z}_{n} \triangleq \{ XX^{T}: X \in   \{ \pm1   \}^{n} \}$. 
Let $\beta \triangleq \lim_{n \rightarrow \infty} \frac{T_{2}}{\log n}$, where $\beta \geq 0$. Also, for convenience define
\begin{equation*}
\eta(a,\beta) \triangleq a-\frac{\gamma}{T_{1}} +\frac{\beta}{2T_{1}} \log  \Bigg( \frac{ (1-\xi) (\gamma + \beta)}{\xi (\gamma - \beta)}  \Bigg) ,
\end{equation*}
where $\gamma \triangleq \sqrt{\beta^{2}+4\xi(1-\xi)a^{2}T_{1}^{2}}$.
\begin{Theorem}
\label{Theorem 3}
Under the binary censored block model and noisy labels, if
\begin{equation*}
\begin{cases}
\eta(a,\beta)>1 &   \text{when } 0\leq \beta<aT_{1}(1-2\xi)\\
\beta>1 & \text{when } \beta \geq aT_{1}(1-2\xi)
\end{cases}
\end{equation*}
then the \ac{SDP} estimator is asymptotically optimal, i.e., $\mathbb{P}(\Zsdp=Z^{*})\geq 1- o(1)$.
\end{Theorem}
\begin{proof}
See Appendix~\ref{Proof-Theorem-3}.
\end{proof}

\begin{Theorem} 
\label{Theorem 4}
Under the binary censored block model and noisy labels, if
\begin{equation*}
\begin{cases}
\eta(a,\beta)<1 &   \text{when } 0\leq \beta<aT_{1}(1-2\xi)\\
\beta<1 & \text{when } \beta \geq aT_{1}(1-2\xi)
\end{cases}
\end{equation*}
then for any sequence of estimators $\widehat{Z}_{n}$, $\mathbb{P}(\widehat{Z}_{n}=Z^{*}) \rightarrow 0$ as $n \rightarrow \infty $.
\end{Theorem}
\begin{proof}
See Appendix~\ref{Proof-Theorem-4}.
\end{proof}

%%%%%%%%%%%%%%%%%%%%%%%%%%%%%%%%%%%%%%%%%%%%%%%%%%%%%%%%%%%%%%%%
\subsection{Censored Block Model with General Side Information}
Combining~\eqref{BCBM-equ1} and~\eqref{G-equ1}, the \ac{SDP} relaxation is 
\begin{align}
\label{BCBM-G-equ3}
\Zsdp=&\underset{Z,X}{\arg\max} ~T_{1}   \langle G, Z   \rangle + 2X^{T}\tilde{Y}  \nonumber\\
&\text{subject to} \quad \begin{bmatrix} 1 & X^{T}\\  X & Z \end{bmatrix} \succeq 0 \nonumber\\
& \quad \quad \quad \quad \quad Z_{ii}=1,\quad i\in [n] .
\end{align}
The log-likelihoods and the log-likelihood-ratio of side information, combined over all features, are as follows:
\begin{align*}
f_{1}(n) &\triangleq \sum_{k=1}^{K} \log \frac{\alpha_{+,m_{k}}^{k}}{\alpha_{-,m_{k}}^{k}} ,\\
f_{2}(n) &\triangleq \sum_{k=1}^{K} \log \alpha_{+,m_{k}}^{k}, \\
f_{3}(n) &\triangleq \sum_{k=1}^{K} \log \alpha_{-,m_{k}}^{k} .
\end{align*}
%	 $f_{1}(n) \triangleq \sum_{k=1}^{K} \log  ( \frac{\alpha_{+,j}^{k}}{\alpha_{-,j}^{k}}  )$, $f_{2}(n) \triangleq \sum_{k=1}^{K} \log \alpha_{+,j}^{k}$ and $f_{3}(n) \triangleq \sum_{k=1}^{K} \log \alpha_{-,j}^{k}$.
Two exponential orders will feature prominently in the following results and proofs:
\begin{align*}
&\beta_{1} \triangleq \lim_{n\rightarrow \infty} \frac{f_{1}(n)}{\log n} ,\\
&\beta \triangleq 
\lim_{n\rightarrow\infty} -\frac{\max (f_2(n),f_3(n))}{\log n} .
\end{align*}
Although the definition of $\beta$ varies in the context of different models, its role remains the same.  In each case, $\beta$ is a parameter representing the asymptotic quality of side information.\footnote{In each case, $\beta$ is proportional to the exponential order of the likelihood function.}

\begin{Theorem}
\label{Theorem 5}
Under the binary censored block model and general side information, if
\begin{equation*}
\begin{cases}\eta(a,   | \beta_{1}   |)+\beta >1 &   \text{when } | \beta_{1}   | \leq aT_{1} (1-2\xi )\\
| \beta_{1} |+\beta >1 & \text{when } |\beta_1|> aT_{1} (1-2\xi )
\end{cases}
\end{equation*}
then the \ac{SDP} estimator is asymptotically optimal, i.e., $\mathbb{P}(\Zsdp=Z^{*})\geq 1- o(1)$.
\end{Theorem}
\begin{proof}
See Appendix~\ref{Proof-Theorem-5}.
\end{proof}

\begin{Theorem} 
\label{Theorem 6}
Under the binary censored block model and general side information, if
\begin{equation*}
\begin{cases}\eta(a,   | \beta_{1}   |)+\beta <1 &   \text{when } | \beta_{1}   | \leq aT_{1} (1-2\xi )\\
| \beta_{1} |+\beta <1 & \text{when } |\beta_1|> aT_{1} (1-2\xi )
\end{cases}
\end{equation*}
then for any sequence of estimators $\widehat{Z}_{n}$, $\mathbb{P}(\widehat{Z}_{n}=Z^{*}) \rightarrow 0$.
\end{Theorem}
\begin{proof}
See Appendix~\ref{Proof-Theorem-6}.
\end{proof}

%%%%%%%%%%%%%%%%%%%%%%%%%%%%%%%%%%%%%%%%%%%%%%%%%%%%%%
\subsection{Stochastic Block Model with Partially Revealed Labels}
Similar to the binary censored block model with partially revealed labels, by combining~\eqref{BSSBM-equ1} and~\eqref{P-equ1}, the \ac{SDP} relaxation is
\begin{align}
\label{BSSBM-P-equ2}
\Zsdp=&\underset{Z}{\arg\max} ~\langle Z,G\rangle \nonumber\\
&\text{subject to} \quad Z\succeq 0 \nonumber\\
& \quad \quad \quad \quad \quad Z_{ii}=1,\quad i\in [n] \nonumber\\
& \quad \quad \quad \quad \quad \langle \mathbf{J} , Z \rangle  = 0 \nonumber\\
& \quad \quad \quad \quad \quad \langle Z,W\rangle =(Y^{T}Y)^2, 
\end{align}
where the constraint $\langle \mathbf{J} , Z \rangle  = 0$ arises from two equal-sized communities. 
\begin{Theorem} 
\label{Theorem 7}
Under the binary symmetric stochastic block model and partially revealed labels, if 
\begin{equation*}
\big(\sqrt{a}-\sqrt{b}\big)^2+2\beta>2 , 
\end{equation*}
then the \ac{SDP} estimator is asymptotically optimal, i.e., $\mathbb{P}(\Zsdp=Z^{*})\geq 1-o(1)$.
\end{Theorem}
\begin{proof}
See Appendix~\ref{Proof-Theorem-7}.
\end{proof}

\begin{Remark}
The converse is given by~\cite[Theorem 3]{Ref7}.
\end{Remark}

\subsection{Stochastic Block Model with Noisy Labels}
Similar to the binary censored block model with noisy labels, by combining~\eqref{BSSBM-equ1} and~\eqref{N-equ1}, the \ac{SDP} relaxation is
\begin{align}
\label{BSSBM-N-equ1}
\Zsdp=&\underset{Z,X}{\arg\max} ~T_{1}   \langle G, Z   \rangle + 2T_{2} X^{T}Y  \nonumber\\
&\text{subject to} \quad \begin{bmatrix} 1 & X^{T}\\  X & Z \end{bmatrix} \succeq 0 \nonumber\\
& \quad \quad \quad \quad \quad Z_{ii}=1,\quad i\in [n] \nonumber\\
& \quad \quad \quad \quad \quad \langle \mathbf{J} , Z \rangle  = 0 . 
\end{align}
For convenience let 
\begin{equation*}
\eta(a,b,\beta) \triangleq  \frac{a+b}{2}+\frac{\beta}{2}-\frac{\gamma}{T_{1}} +\frac{\beta}{2T_{1}} \log  \bigg( \frac{\gamma +\beta}{\gamma -\beta}  \bigg) ,
\end{equation*}
where $\gamma \triangleq \sqrt{\beta^{2}+abT_{1}^{2}}$.
\begin{Theorem}
\label{Theorem 9}
Under the binary symmetric stochastic block model and noisy label side information, if
\begin{equation*}
\begin{cases}
\eta(a,b, \beta)>1 &   \text{when } 0\leq \beta<\frac{T_{1}}{2}(a-b)\\
\beta>1 & \text{when } \beta \geq \frac{T_{1}}{2}(a-b)
\end{cases}
\end{equation*}
then the \ac{SDP} estimator is asymptotically optimal, i.e., $\mathbb{P}(\Zsdp=Z^{*})\geq 1- o(1)$.
\end{Theorem}
\begin{proof}
See Appendix~\ref{Proof-Theorem-9}.
\end{proof}

\begin{Remark}
The converse is given by~\cite[Theorem 2]{Ref7}.
\end{Remark}
%%%%%%%%%%%%%%%%%%%%%%%%%%%%%%%%%%%%%%%%%%%%
\subsection{Stochastic Block Model with General Side Information}
Similar to the binary censored block model with general side information, by combining~\eqref{BSSBM-equ1} and~\eqref{G-equ1}, the \ac{SDP} relaxation is
\begin{align}
\label{BSSBM-G-equ1}
\Zsdp=&\underset{Z,X}{\arg\max} ~T_{1}   \langle G, Z   \rangle + 2X^{T}\tilde{Y}  \nonumber\\
&\text{subject to} \quad \begin{bmatrix} 1 & X^{T}\\  X & Z \end{bmatrix} \succeq 0 \nonumber\\
& \quad \quad \quad \quad \quad Z_{ii}=1,\quad i\in [n] \nonumber\\
& \quad \quad \quad \quad \quad \langle \mathbf{J} , Z \rangle  = 0 .
\end{align}

\begin{Theorem}
\label{Theorem 11}
Under the binary symmetric stochastic block model and general side information, if 
\begin{equation*}
\begin{cases}
\eta(a,b,   | \beta_{1}   |)+\beta >1 &   \text{when } | \beta_{1}   | \leq T_{1}\frac{(a-b)}{2} \\
| \beta_{1}   |+\beta >1 & \text{when } | \beta_{1}   | > T_{1}\frac{(a-b)}{2}
\end{cases}
\end{equation*}
then the \ac{SDP} estimator is asymptotically optimal, i.e., $\mathbb{P}(\Zsdp=Z^{*})\geq 1- o(1)$.
\end{Theorem}
\begin{proof}
See Appendix~\ref{Proof-Theorem-11}.
\end{proof}

\begin{Remark}
The converse is given by~\cite[Theorem 5]{Ref7}.
\end{Remark}

\section{Numerical Results}
\label{Section: Numerical Results}

This section produces numerical simulations that shed light on the domain of applicability of the asymptotic results obtained earlier in the paper\footnote{The code is available online at \url{https://github.com/mohammadesmaeili/Community-Detection-by-SDP} }.

Table~\ref{Table 1} shows the misclassification error probability of the \ac{SDP} estimators~\eqref{BCBM-P-equ2} and \eqref{BSSBM-P-equ2} with partially revealed side information. 
Under the binary stochastic block model with $a=3$ and $b=1$, when the side information $\beta =  0.8$, error probability diminishes with $n$ as predicted by earlier asymptotic results. For these parameters, $\eta = 1.1 > 1$, and exact recovery is possible based on the theoretical results.   
When $\beta=0.2$,  then $\eta=0.5 < 1$ which does not fall in the asymptotic perfect recovery regime, the misclassification error probability is much higher.
Under the binary censored block model with $a=1$ and $\xi=0.2$, when the side information $\beta =  1$, error probability diminishes with $n$.
For these values, $\eta =1.2 > 1$, and exact recovery is possible based on the theoretical results.   
When $\beta=0.3$, the misclassification error probability is much higher. For this value of $\beta$, $\eta=0.5 < 1$ which means exact recovery is not asymptotically  possible.

\begin{table*}
\begin{minipage}[t]{0.33\textwidth}
\centering
\caption{SDP with partially revealed labels.}
\begin{tabular}{@{}cccccc@{}}
\toprule
a  & b  & $\xi$                  & $\beta$                  & n    & Error Probability    \\ \midrule
3 & 1 & - & 0.2 & 100 & $4.1 \times 10^{-2}$ \\
3 & 1 & - & 0.2 & 200 & $3.1 \times 10^{-2}$ \\
3 & 1 & - & 0.2 & 300 & $2.5 \times 10^{-2}$ \\
3 & 1 & - & 0.2 & 400 & $2.2 \times 10^{-2}$ \\
3 & 1 & - & 0.2 & 500 & $1.9 \times 10^{-2}$ \\
3 & 1 & - & 0.8 & 100 & $5.0 \times 10^{-4}$ \\
3 & 1 & - & 0.8 & 200 & $3.2 \times 10^{-4}$ \\
3 & 1 & - & 0.8 & 300 & $1.6 \times 10^{-4}$ \\
3 & 1 & - & 0.8 & 400 & $1.2 \times 10^{-4}$ \\
3 & 1 & - & 0.8 & 500 & $9.3 \times 10^{-5}$ \\ %\midrule
1  & - & 0.2  & 0.3    & 100    & $4.1 \times 10^{-2}$ \\
1  & - & 0.2  & 0.3    & 200    & $2.9 \times 10^{-2}$ \\
1  & - & 0.2  & 0.3    & 300    & $2.2 \times 10^{-2}$ \\
1  & - & 0.2  & 0.3    & 400    & $1.9 \times 10^{-2}$ \\
1  & - & 0.2  & 0.3    & 500    & $1.7 \times 10^{-2}$ \\
1  & - & 0.2  & 1    & 100    & $1.1 \times 10^{-3}$ \\
1  & - & 0.2  & 1    & 200    & $4.2 \times 10^{-4}$  \\
1  & - & 0.2  & 1    & 300    & $2.7 \times 10^{-4}$ \\
1  & - & 0.2  & 1    & 400    & $2.1 \times 10^{-4}$ \\
1  & - & 0.2  & 1    & 500    & $1.5 \times 10^{-4}$ \\ \bottomrule
\end{tabular}
\label{Table 1}
\end{minipage}
\hfill
\begin{minipage}[t]{0.33\textwidth}
\centering
\caption{SDP with noisy labels}
\begin{tabular}{@{}cccccc@{}}
\toprule
a  & b  & $\xi$                  & $\beta$                  & n    & Error Probability    \\ \midrule
4 & 1 & - & 0.2 & 100 & $2.0 \times 10^{-2}$ \\
4 & 1 & - & 0.2 & 200 & $1.5 \times 10^{-2}$ \\
4 & 1 & - & 0.2 & 300 & $1.3 \times 10^{-2}$ \\
4 & 1 & - & 0.2 & 400 & $1.1 \times 10^{-2}$ \\
4 & 1 & - & 0.2 & 500 & $1.0 \times 10^{-2}$ \\
4 & 1 & - & 1 & 100 & $1.1 \times 10^{-3}$ \\
4 & 1 & - & 1 & 200 & $7.4 \times 10^{-4}$ \\
4 & 1 & - & 1 & 300 & $3.0 \times 10^{-5}$ \\
4 & 1 & - & 1 & 400 & $2.7 \times 10^{-5}$ \\
4 & 1 & - & 1 & 500 & $2.2 \times 10^{-5}$ \\ %\midrule
4  & - & 0.25  & 0.1    & 100    & $2.9 \times 10^{-2}$ \\
4  & - & 0.25  & 0.1    & 200    & $1.8 \times 10^{-2}$ \\
4  & - & 0.25  & 0.1    & 300    & $1.4 \times 10^{-2}$ \\
4  & - & 0.25  & 0.1    & 400    & $1.2 \times 10^{-2}$ \\
4  & - & 0.25  & 0.1    & 500    & $1.0 \times 10^{-2}$ \\
4  & - & 0.25  & 1.1    & 100    & $2.7 \times 10^{-3}$ \\
4  & - & 0.25  & 1.1    & 200    & $1.0 \times 10^{-3}$ \\
4  & - & 0.25  & 1.1    & 300    & $6.2 \times 10^{-4}$ \\
4  & - & 0.25  & 1.1    & 400    & $4.1 \times 10^{-4}$ \\
4  & - & 0.25  & 1.1    & 500    & $3.3 \times 10^{-4}$ \\ \bottomrule
\end{tabular}
\label{Table 2}
\end{minipage}
\hfill
\begin{minipage}[t]{0.33\textwidth}
\centering
\caption{SDP without side information.}
\begin{tabular}{@{}ccccc@{}}
\toprule
a  & b  & $\xi$                  & n    & Error Probability    \\
\midrule
3 & 1 & - & 100 & $1.4 \times 10^{-1}$ \\
3 & 1 & - & 200 & $1.2 \times 10^{-1}$ \\
3 & 1 & - & 300 & $1.1 \times 10^{-1}$ \\
3 & 1 & - & 400 & $9.8 \times 10^{-2}$ \\
3 & 1 & - & 500 & $9.1 \times 10^{-2}$ \\
4 & 1 & - & 100 & $2.3 \times 10^{-2}$ \\
4 & 1 & - & 200 & $1.7 \times 10^{-2}$ \\
4 & 1 & - & 300 & $1.6 \times 10^{-2}$ \\
4 & 1 & - & 400 & $1.3 \times 10^{-2}$ \\
4 & 1 & - & 500 & $1.2 \times 10^{-2}$ \\ %\midrule
1  & - & 0.2  & 100    & $2.9 \times 10^{-1}$ \\
1  & - & 0.2  & 200    & $2.5 \times 10^{-1}$ \\
1  & - & 0.2  & 300    & $2.2 \times 10^{-1}$ \\
1  & - & 0.2  & 400    & $2.1 \times 10^{-1}$ \\
1  & - & 0.2  & 500    & $1.9 \times 10^{-1}$ \\
4  & - & 0.25  & 100    & $3.0 \times 10^{-2}$ \\
4  & - & 0.25  & 200    & $1.9 \times 10^{-2}$ \\
4  & - & 0.25  & 300    & $1.5 \times 10^{-2}$ \\
4  & - & 0.25  & 400    & $1.2 \times 10^{-2}$ \\
4  & - & 0.25  & 500    & $1.1 \times 10^{-2}$ \\ \bottomrule
\end{tabular}
\label{Table 3}
\end{minipage}
\end{table*}

Table~\ref{Table 2} shows the misclassification error probability of the \ac{SDP} estimators~\eqref{BCBM-N-equ3} and \eqref{BSSBM-N-equ1} with noisy labels side information.
Under the stochastic block model with $a=4$ and $b=1$,  when the side information $\beta =  1$, then $\eta =1.1 > 1$ and the error probability diminishes with $n$ as predicted by earlier theoretical results. When $\beta=0.2$,  then $\eta =0.6 < 1$ which does not fall in the asymptotic perfect recovery regime. For this case the misclassification error is much higher.
Under the censored block model
with $a=4$ and $\xi=0.25$,  when the side information $\beta =  1.1$, then $\eta =1.2 > 1$ and the error probability diminishes with $n$. When $\beta=0.1$,  then $\eta =0.6 < 1$ which means that exact recovery is not possible asymptotically. For this value of $\beta$ and a finite $n$, the  misclassification error is not negligible.

For comparison, Table~\ref{Table 3} shows the misclassification error probability of the \ac{SDP} estimator {\em without} side information, i.e., $\beta = 0$.
Under the binary stochastic block model, when $a=3$ ($a=4$) and $b=1$, it is seen that the error probability increases in comparison with the corresponding error probability in Table~\ref{Table 1} (Table~\ref{Table 2}) where side information is available.
Also, under the binary censored block model, when $a=1$ and $\xi=0.2$ ($a=4$ and $\xi=0.25$), it is seen that the error probability increases in comparison with the corresponding error probability in Table~\ref{Table 1} (Table~\ref{Table 2}) where side information is available.

Using standard arguments form numerical linear algebra, the computational complexity of the algorithms in this paper are on the order
$O(mn^3 + m^2n^2)$, where $n$ is the number of nodes in the graph, and $m$ is a small constant, typically between 2 to 4, indicating assumptions of the problem that manifest as constraints in the optimization.

\section{Conclusion}
This paper calculated the exact recovery threshold for community detection under SDP with several types of side information. Among other insights, our results indicate that in the presence of side information, the exact recovery threshold for SDP and for maximum likelihood detection remain identical.  We anticipate that models and methods of this paper may be further extended to better match the statistics of real-world graph data.

\appendices

%%%%%%%%%%%%%%%%%%%%%%%%%%%%%%%%%%%%%%%%%%%%%%%%%%%%%%%%%%%%%%%%%%%%%%%%%%%%%%%%%%%%%%%%%%%%%%%%%%%%%%%%%%%%%%%%%%%
%\section*{APPENDIX}
%
%Appendixes should appear before the acknowledgment.
%
%\section*{ACKNOWLEDGMENT}
%
%The preferred spelling of the word ÒacknowledgmentÓ in America is without an ÒeÓ after the ÒgÓ. Avoid the stilted expression, ÒOne of us (R. B. G.) thanks . . .Ó  Instead, try ÒR. B. G. thanksÓ. Put sponsor acknowledgments in the unnumbered footnote on the first page.

%%%%%%%%%%%%%%%%%%%%%%%%%%%%%%%%%%%%%%%%%%%%%%%%%%%%%%%%%%%%%%%%%%%%%%%%%%%%%%%%%%%%%%%%%%%%%%%%%%%%%%%%%%%%%%%%%%
\section{Proof of Theorem~\ref{Theorem 1}}
\label{Proof-Theorem-1}
We begin by stating sufficient conditions for the optimum solution of~\eqref{BCBM-P-equ2} matching the true labels $X^*$.

\begin{Lemma}
\label{BCBM-P-Lemma 1}
For the optimization problem~\eqref{BCBM-P-equ2}, consider the Lagrange multipliers
\begin{equation*}
\mu^* , \quad D^{*}=\mathrm{diag}(d_{i}^{*}), \quad 
S^{*}.
\end{equation*}
If we have
\begin{align*}
&S^{*} = D^{*}+\mu^{*}W-G ,\\
&S^{*} \succeq 0, \\
&\lambda_{2}(S^{*}) >  0 ,\\
&S^{*}X^{*} =0 ,
\end{align*}
then $(\mu^{*}, D^*, S^*)$ is the dual optimal solution and $\Zsdp=X^{*}X^{*T}$ is the unique primal optimal solution of~\eqref{BCBM-P-equ2}.
\end{Lemma}

\begin{proof}
The Lagrangian of~\eqref{BCBM-P-equ2} is given by
\begin{align*}
L(Z,S,D,\mu)=&\langle G,Z \rangle +\langle S, Z \rangle -\langle D,Z-\mathbf{I} \rangle \\
&- \mu  ( \langle W,Z\rangle - (Y^{T}Y  )^{2}  ) ,
\end{align*}
where $S\succeq 0$, $D=\mathrm{diag}(d_{i})$, and $\mu \in \mathbb{R}$ are Lagrange multipliers. For any $Z$ that satisfies the constraints in~\eqref{BCBM-P-equ2},
\begin{align*}
\langle G,Z\rangle & \overset{(a)}{\leq} L(Z,S^{*},D^{*},\mu^{*})\\
&=\langle D^{*},\mathbf{I}\rangle +\mu^{*}(Y^{T}Y)^{2}\\
&\overset{(b)}{=}\langle D^{*},Z^{*}\rangle +\mu^{*}(Y^{T}Y)^{2}\\
&=\langle G+S^{*}-\mu^{*} W,Z^{*}\rangle +\mu^{*}(Y^{T}Y)^{2}\\
&\overset{(c)}{=}\langle G,Z^{*}\rangle ,
\end{align*}
where $(a)$ holds because $\langle S^{*},Z \rangle \geq 0$, $(b)$ holds because $Z_{ii}=1$ for all $i \in [n]$, and $(c)$ holds because $\langle S^{*},Z^{*} \rangle = X^{*T}S^{*}X^{*}=0$ and $\langle W,Z^{*} \rangle = (Y^{T}Y)^{2}$. Therefore, $Z^{*}$ is a primal optimal solution. Now, we will establish the uniqueness of the optimal solution. Assume $\tilde{Z}$ is another primal optimal solution. Then
\begin{align*}
\langle S^{*},\tilde{Z}\rangle & =\langle D^{*}-G+\mu^{*}W , \tilde{Z} \rangle \\
&=\langle D^{*},\tilde{Z} \rangle-\langle G,\tilde{Z} \rangle+ \mu^{*} \langle W,\tilde{Z} \rangle \\
&\overset{(a)}{=}\langle D^{*},Z^{*}\rangle-\langle G,Z^{*}\rangle+ \mu^{*} \langle W,Z^{*}\rangle \\
&=\langle D^{*}-G+\mu^{*}W,Z^{*}\rangle \\
&=\langle S^{*}, Z^{*}\rangle=0 ,
\end{align*}
where $(a)$ holds because $\langle W,Z^{*}\rangle=\langle W,\tilde{Z} \rangle=(Y^{T}Y)^{2}$, $\langle G,Z^{*}\rangle=\langle G,\tilde{Z} \rangle$, and $Z_{ii}^{*}=\tilde{Z}_{ii}=1$ for all $i\in [n]$. Since $\tilde{Z}\succeq 0$ and $S^{*}\succeq 0$ while its second smallest eigenvalue $\lambda_{2}(S^{*})$ is positive, $\tilde{Z}$ must be a multiple of $Z^{*}$. Also, since $\tilde{Z}_{ii}=Z_{ii}^{*}=1$ for all $i \in [n]$, we have $\tilde{Z}=Z^{*}$.
\end{proof}

We now show that $S^{*} = D^{*}+\mu^{*}W-G$ satisfies other conditions in Lemma~\ref{BCBM-P-Lemma 1} with probability $1-o(1)$. 
Let 
\begin{equation}
\label{BCBM-P-equ3}
d_{i}^{*}= \sum_{j=1}^{n} G_{ij}x_{j}^{*}x_{i}^{*} -\mu^{*} \sum_{j=1}^{n} y_{i}y_{j}x_{j}^{*}x_{i}^{*} .
\end{equation}
Then $D^{*}X^{*} = GX^{*}-\mu^{*}WX^{*}$ and based on the definition of $S^{*}$ in  Lemma~\ref{BCBM-P-Lemma 1}, $S^{*}$ satisfies the condition $S^{*}X^{*} =0$.
It remains to show that $S^{*}\succeq 0$ and $\lambda_{2}(S^{*})>0$ with probability $1-o(1)$. In other words, we need to show that
\begin{equation}
\label{BCBM-P-equ1-New}
\mathbb{P}  \bigg( \underset{V\perp X^{*},  \| V   \|=1}{\inf} V^{T}S^{*}V>0   \bigg )\geq 1-o(1) , 
\end{equation}
where $V$ is a vector of length $n$. Since for the binary censored block model
\begin{align}
\label{BCBM-equ-expectation}
    \mathbb{E}[G]=p(1-2\xi)(X^{*}X^{*T}-\mathbf{I}) ,
\end{align}
it follows that for any $V$ such that $V^{T}X^{*}=0$ and $  \| V   \|=1$,
\begin{align*}
V^{T}S^{*}V=&V^{T}D^{*}V +\mu^{*}V^{T}WV  -V^{T}(G-\mathbb E[G])V \\
& +p(1-2\xi).
\end{align*}
\begin{Lemma}\cite[Thoerem 9]{Ref19}
\label{BCBM-P-Lemma 2}
For any $c > 0$, there exists $c' >0$ such that for any $n \geq 1$, $  \| G-\mathbb E[G]   \| \leq c'\sqrt{\log n}$ with probability at least $1-n^{-c}$. 
\end{Lemma}
\begin{Lemma}\cite[Lemma 3]{Esmaeili.BSSBM.Partially.Revealed}
\label{BCBM-P-Lemma 3}
\begin{equation*}
    \mathbb{P}   \Big ( V^{T}WV\geq \sqrt{\log n}   \Big ) \leq \frac{1-\epsilon }{\sqrt{\log n}} = n^{-\frac{1}{2}+o(1)} .
\end{equation*}
\end{Lemma}

Since $V^{T}D^{*}V \geq \min_{i\in [n]} d_{i}^{*}$ and $V^{T}(G-\mathbb E[G])V \leq    \| G-\mathbb E[G]   \| $, applying Lemmas~\ref{BCBM-P-Lemma 2} and~\ref{BCBM-P-Lemma 3} implies that with probability $1-o(1)$,
\begin{equation}
\label{BCBM-P-equ5}
V^{T}S^{*}V \geq \min_{i \in [n]} d_{i}^{*} +(\mu^{*}-c^{'}) \sqrt{\log n} +p(1-2\xi) .
\end{equation}
\begin{Lemma}
\label{BCBM-P-Lemma 4}
Consider a sequence of i.i.d. random variables $\{S_1,\ldots,S_m\}$ with distribution $p(1-\xi) \delta_{+1} + p\xi \delta_{-1} +(1-p) \delta_{0}$. Let $U \sim \text{Binom} (n-1,1-\epsilon)$, $\mu^{*}<0$, and $\delta= \frac{\log n}{\log log n}$. Then
\begin{align*}
&\mathbb{P} \Bigg(\sum_{i=1}^{n-1} S_{i} \leq \delta \Bigg)  \leq n^{-a (\sqrt{1-\xi}-\sqrt{\xi} )^{2}+o(1) } , \\
&\mathbb{P} \Bigg(\sum_{i=1}^{n-1} S_{i}-\mu^{*}U  \leq \delta+\mu^{*} \Bigg) \leq \epsilon^{n [\log \epsilon +o(1) ] } .
\end{align*}

\end{Lemma}
\begin{proof}
It follows from Chernoff bound. 
\end{proof}

It can be shown that $\sum_{j=1}^{n} G_{ij}x_{i}^{*}x_{j}^{*}$ in~\eqref{BCBM-P-equ3} is equal in distribution to $\sum_{i=1}^{n-1} S_{i}$ in Lemma~\ref{BCBM-P-Lemma 4}. Then
\begin{align*}
\mathbb{P} ( d_{i}^{*}\leq \delta  ) = &\mathbb{P}  \Bigg( \sum_{j=1}^{n} G_{ij}x_{i}^{*}x_{j}^{*} \leq \delta   \Bigg) \epsilon \\
&+ \mathbb{P}  \Bigg( \sum_{j=1}^{n} G_{ij}x_{i}^{*}x_{j}^{*}-\mu^{*}Z_{i} \leq \delta+\mu^{*}   \Bigg) (1-\epsilon) \\
\leq & \epsilon n^{-a (\sqrt{1-\xi}-\sqrt{\xi} )^{2}+o(1) } +  (1-\epsilon  ) \epsilon^{n \big( \log \epsilon +o(1) \big) } \\
= & e^{  \big( \frac{\log \epsilon }{\log n} -a (\sqrt{1-\xi}-\sqrt{\xi} )^{2}+o(1)  \big) \log n} ,
\end{align*}
where $Z_{i} \sim \text{Binom} (n-1,1-\epsilon)$ and  $ (1-\epsilon  ) \epsilon^{n (\log \epsilon +o(1)) }$ vanishes as $n  \rightarrow \infty$. 
% Let $\log \epsilon  = -\beta \log n +o(\log n)$ and $\beta \geq 0$.
Recall that $\beta \triangleq \lim_{n \rightarrow \infty} -\frac{\log \epsilon}{\log n}$, where $\beta \geq 0$.
Then
\begin{equation*}
    \mathbb{P} ( d_{i}^{*}\leq \delta  ) \leq n^{ -\beta -a (\sqrt{1-\xi}-\sqrt{\xi} )^{2}+o(1) }.
\end{equation*}
Using the union bound, 
\begin{equation*}
\mathbb{P} \bigg( \min_{i \in [n]}d_{i}^{*} \geq \frac{\log n}{\log \log n}  \bigg ) \geq 1-n^{1-\beta-a(\sqrt{1-\xi}-\sqrt{\xi})^{2}+o(1)} .
\end{equation*}
When $\beta+ a(\sqrt{1-\xi}-\sqrt{\xi})^{2}>1 $, $\min_{i \in [n]}d_{i}^{*} \geq \frac{\log n}{\log \log n}$ holds with probability $1-o(1)$. Combining this result with~\eqref{BCBM-P-equ5},  if $\beta + a(\sqrt{1-\xi}-\sqrt{\xi})^{2}>1 $, then with probability $1-o(1)$,
\begin{equation*}
V^{T}S^{*}V \geq \frac{\log n}{\log \log n} +(\mu^{*}-c^{'}) \sqrt{\log n} +p(1-2\xi) > 0 ,
\end{equation*}
which concludes Theorem~\ref{Theorem 1}.

%%%%%%%%%%%%%%%%%%%%%%%%%%%%%%%%%%%%%%%%%%%%%%%%%%%%%%%%%%%%%%%%%%%%%%%%%%%%%%%%%%%%%%%%%%%%%%%%%%%%%%%%%%%%%%%%
\section{Proof of Theorem~\ref{Theorem 2}}
\label{Proof-Theorem-2}
Since the prior distribution of $X^{*}$ is uniform, among all estimators, the maximum likelihood estimator minimizes the average error probability. Therefore, it suffices to show that with high probability the maximum likelihood estimator fails. 
Let 
\[
F \triangleq \bigg\{\min_{i \in [n], y_{i}=0}~\sum_{j =1}^{n} G_{ij} x_{j}^{*} x_{i}^{*} \leq -1 \bigg\}.
\]
Then $\mathbb{P} ( \text{ML Fails}  ) \geq \mathbb{P} ( F  )$. If we show that $\mathbb{P} ( \text{F}  )  \rightarrow 1$, the maximum likelihood estimator fails. Let $H$ denote the set of first $  \lfloor \frac{n}{\log^{2} n}   \rfloor$ nodes and $e (i, H  )$ denote the number of edges between node $i$ and nodes in the set $H$. Then
\begin{align*}
\min_{i \in [n], y_{i}=0} & ~\sum_{j =1}^{n} G_{ij} x_{j}^{*} x_{i}^{*} \leq  \min_{i \in H, y_{i}=0 }~\sum_{j =1}^{n} G_{ij} x_{j}^{*} x_{i}^{*} \\
\leq & \min_{i \in H, y_{i}=0 }~\sum_{ j \in H^{c} } G_{ij} x_{j}^{*} x_{i}^{*} + \max_{i \in H, y_{i}=0 }~e (i, H  ) ,
\end{align*}
Define the events
\begin{align*}
&E_1 \triangleq \bigg\{\max_{i \in H, y_{i}=0 }~e (i, H  ) \leq \delta -1 \bigg\}, \\
&E_2 \triangleq\Bigg\{\min_{i \in H, y_{i}=0}~\sum_{ j \in H^{c} } G_{ij} x_{j}^{*} x_{i}^{*} \leq -\delta \Bigg\}.
\end{align*}
Notice that $F \supset E_{1} \cap E_{2}$. Hence, to show that the maximum likelihood estimator fails, it suffices to show that $\mathbb{P} ( E_{1}  )  \rightarrow 1$ and $\mathbb{P} ( E_{2}  )  \rightarrow 1$.
\begin{Lemma}\cite[Lemma 5]{esmaeili2019community}
\label{BCBM-P-Lemma 5}
When $S \sim \text{Binom}(n,p)$, for any $r\geq 1$, $\mathbb{P} ( S \geq rnp  ) \leq  \big(\frac{e}{r} \big)^{rnp} e^{-np}$. 
\end{Lemma}

Since $e (i, H  ) \sim \text{Binom} \big(  | H   |, a\frac{\log n }{n} \big)$, it follows from Lemma~\ref{BCBM-P-Lemma 5} that
\begin{align*}
\mathbb{P} &  \big( e (i, H  ) \geq \delta-1 , y_{i}=0  \big) \\
&\leq \epsilon  \bigg( \frac{\log^{2} n}{ae \log \log n} - \frac{\log n}{ae}  \bigg)^{1-\frac{\log n}{\log \log n}} e^{-\frac{a}{\log n}} \leq  \epsilon n^{-2+o(1)} .
\end{align*}
Using the union bound,  $\mathbb{P}  ( E_{1}  ) \geq 1-\epsilon n^{-1+o(1)}$. Thus, $\mathbb{P}  ( E_{1}  ) \to 1$.

\begin{Lemma}\cite[Lemma 8]{Ref19}
\label{BCBM-P-Lemma 6}
Consider a sequence of i.i.d. random variables $\{S_1,\ldots,S_m\}$ with distribution $ p(1-\xi) \delta_{+1} + p\xi \delta_{-1} +(1-p) \delta_{0}$, where $m-n = o(n)$. Let $f(n) = \frac{\log n}{\log \log n}$. Then
\begin{equation*}
\mathbb{P}   \Bigg( \sum_{i=1}^{m} S_{i} \leq -f(n)  \Bigg) \geq n^{-a ( \sqrt{1-\xi}-\sqrt{\xi}  )^{2}+o(1) } .
\end{equation*}
\end{Lemma}
Using Lemma~\ref{BCBM-P-Lemma 6} and since $  \{ \sum_{ j \in H^{c} } G_{ij} x_{j}^{*} x_{i}^{*}    \}_{i \in H}$ are mutually independent, 
\begin{align}
\label{BCBM-P-equ9}
\mathbb{P}  ( E_{2}  ) & =1 - \prod_{i\in H}  \Bigg[ 1- \mathbb{P}  \bigg( \sum_{ j \in H^{c} } G_{ij} x_{j}^{*} x_{i}^{*} \leq -\delta, y_{i}=0  \bigg)  \Bigg]  \nonumber\\
&\geq 1 -  \Big[ 1- \epsilon n^{-a ( \sqrt{1-\xi}-\sqrt{\xi}  )^{2}+o(1)}  \Big]^{  | H   |} .
\end{align}
Since $\beta = \lim_{n \rightarrow \infty} -\frac{\log \epsilon}{\log n}$, it follows from~\eqref{BCBM-P-equ9} that
\begin{align}
\label{BCBM-P-equ11}
\mathbb{P}  ( E_{2}  ) &\geq 1 - \Big[ 1- n^{-\beta -a ( \sqrt{1-\xi}-\sqrt{\xi}  )^{2}+o(1)} \Big]^{  | H   |} \nonumber\\
&\geq 1 - \exp  \Big( - n^{1-\beta -a ( \sqrt{1-\xi}-\sqrt{\xi}  )^{2}+o(1)}  \Big) ,
\end{align}
using $1+x \leq e^{x}$. From~\eqref{BCBM-P-equ11}, if $a ( \sqrt{1-\xi}-\sqrt{\xi}  )^{2}+\beta <1$, then $\mathbb{P}  ( E_{2}  )  \rightarrow 1$. Therefore, $\mathbb{P}  ( F  )  \rightarrow 1$ and Theorem~\ref{Theorem 2} follows.

%%%%%%%%%%%%%%%%%%%%%%%%%%%%%%%%%%%%%%%%%%%%%%%%%%%%%%%%%%%%%%%%%%%%%%%%%%%%%%%%%%%%%%%%%%%%%%%%%%%%%%%%%%
\section{Proof of Theorem~\ref{Theorem 3}}
\label{Proof-Theorem-3}
We begin by deriving sufficient conditions for the \ac{SDP} estimator to produce the true labels $X^*$.

\begin{Lemma}
\label{BCBM-N-Lemma 1}
For the optimization problem~\eqref{BCBM-N-equ3}, consider the Lagrange multipliers
\begin{equation*}
D^{*}=\mathrm{diag}(d_{i}^{*}), \qquad 
S^{*}\triangleq \begin{bmatrix} S_{A}^{*} & S_{B}^{*T} \\ S_{B}^{*} & S_{C}^{*} \end{bmatrix}.
\end{equation*}
If we have
\begin{align*}
&S_{A}^{*} =  T_{2}Y^{T}X^{*} , \\
&S_{B}^{*} = -T_{2}Y , \\
&S_{C}^{*} = D^{*}-T_{1}G, \\
&S^{*} \succeq 0, \\
&\lambda_{2}(S^{*}) >  0, \\
&S^{*} [1, X^{*T}]^T =0
\end{align*}
then $(D^* , S^*)$ is the dual optimal solution and $\Zsdp=X^{*}X^{*T}$ is the unique primal optimal solution of~\eqref{BCBM-N-equ3}.
\end{Lemma}

\begin{proof}
Define
\begin{equation*}
    H \triangleq  \begin{bmatrix} 1 & X^{T}\\  X & Z \end{bmatrix}.
\end{equation*}
The Lagrangian of~\eqref{BCBM-N-equ3} is given by
\begin{equation*}
L(Z,X,S,D)=T_{1} \langle G,Z \rangle +2T_{2} \langle Y, X \rangle +\langle S, H \rangle -\langle D,Z-\mathbf{I} \rangle ,
\end{equation*}
where $S\succeq 0$ and $D=\mathrm{diag}(d_{i})$ are Lagrange multipliers. For any $Z$ that satisfies the constraints in~\eqref{BCBM-N-equ3}, 
\begin{align*}
T_{1} \langle G,Z \rangle+2T_{2} \langle Y,X \rangle &\overset{(a)}{\leq} L(Z,X,S^{*},D^{*})\\
&=\langle D^{*},\mathbf{I}\rangle +S_{A}^{*}\\
&\overset{(b)}{=}\langle D^{*},Z^{*}\rangle -\langle S_{B}^{*} , X^{*} \rangle \\
&=\langle S_{C}^{*}+T_{1}G,Z^{*}\rangle - \langle S_{B}^{*} , X^{*} \rangle \\
&\overset{(c)}{=}T_{1}\langle G, Z^{*}\rangle -2\langle S_{B}^{*} , X^{*} \rangle \\
&\overset{(d)}{=}T_{1}\langle G, Z^{*}\rangle + 2T_{2}\langle Y , X^{*} \rangle ,
\end{align*}
where $(a)$ holds because $\langle S^{*},H \rangle \geq 0$, $(b)$ holds because $Z_{ii}=1$ for all $i \in [n]$ and $S_{A}^{*}=-S_{B}^{*T}X^{*}$, $(c)$ holds because $S_{B}^{*}=-S_{C}^{*}X^{*}$, and $(d)$ holds because $S_{B}^{*}=-T_{2}Y$. Therefore, $Z^{*}=X^{*}X^{*T}$ is a primal optimal solution. Now, assume $\tilde{Z}$ is another optimal solution. 
\begin{align*}
\langle S^{*}, & \tilde{H} \rangle =S_{A}^{*} + 2\langle S_{B}^{*},\tilde{X} \rangle  + \langle D^{*} -T_{1}G , \tilde{Z} \rangle \\
&\overset{(a)}{=}S_{A}^{*} + 2\langle S_{B}^{*},X^{*} \rangle +\langle D^{*},Z^{*}\rangle -T_{1}\langle G,Z^{*}\rangle\\
&=\langle S^{*}, H^{*} \rangle=0
\end{align*}
where $(a)$ holds because $\langle G,Z^{*}\rangle=\langle G,\tilde{Z} \rangle$, $Z_{ii}^{*}=\tilde{Z}_{ii}=1$ for all $i\in [n]$, and $\langle S_{B}^{*},X^{*} \rangle = \langle S_{B}^{*},\tilde{X} \rangle$. Since $\tilde{H} \succeq 0$ and $S^{*}\succeq 0$ while its second smallest eigenvalue $\lambda_{2}(S^{*})$ is positive, $\tilde{H}$ must be a multiple of $H^{*}$. Also, since $\tilde{Z}_{ii}=Z_{ii}^{*}=1$ for all $i \in [n]$, we have $\tilde{H}=H^{*}$.
\end{proof}

We now show that $S^{*}$ defined by $S_{A}^{*}$, $S_{B}^{*}$, and $S_{C}^{*}$ satisfies other conditions in Lemma~\ref{BCBM-N-Lemma 1} with probability $1-o(1)$. 
Let 
\begin{equation}
\label{BCBM-N-equ44}
d_{i}^{*}=T_{1} \sum_{j=1}^{n} G_{ij}x_{j}^{*}x_{i}^{*} + T_{2}y_{i}x_{i}^{*}.
\end{equation}
Then $D^{*}X^{*} = T_{1}GX^{*}+T_{2}Y$ and based on the definitions of $S_{A}^{*}$, $S_{B}^{*}$, and $S_{C}^{*}$ in  Lemma~\ref{BCBM-N-Lemma 1}, $S^{*}$ satisfies the condition $S^{*} [1, X^{*T}]^T =0$.
It remains to show that $S^{*}\succeq 0$ and $\lambda_{2}(S^{*})>0$ with probability  $1-o(1)$. In other words, we need to show that
\begin{equation}
\label{BCBM-N-equ1 New}
\mathbb{P}  \bigg( \underset{V\perp [1, X^{*T}]^T,  \| V   \|=1}{\inf} V^{T}S^{*}V>0   \bigg) \geq 1-o(1) ,
\end{equation}
where $V$ is a vector of length $n+1$. Let $V \triangleq [v, U^{T}]^{T}$, where $v$ is a scalar and $U \triangleq [u_{1}, u_{2}, \cdots,u_{n}]^{T}$. For any $V$ such that $V^{T}[1, X^{*T}]^T=0$ and $  \| V   \|=1$, we have
\begin{align}
\label{BCBM-N-equ4}
V&^{T}S^{*}V =v^{2} S_{A}^{*} -2T_{2}vU^{T}Y +U^{T}D^{*}U - T_{1}U^{T}GU \nonumber\\
\geq&   ( 1-v^{2}   )  \bigg[\min_{i \in [n]} d_{i}^{*} - T_{1}  \| G-\mathbb{E}[G]   \| + T_{1} p(1-2\xi) \bigg] \nonumber\\
&+v^{2}  \bigg[ T_{2}Y^{T}X^{*} -2T_{2}\frac{\sqrt{n(1-v^{2})}}{|v|}-T_{1}p(1-2\xi)   \bigg] ,
\end{align}
where the last inequality holds because
\begin{equation*}
    U^{T}D^{*}U \geq   ( 1-v^{2}   ) \min_{i \in [n]} d_{i}^{*} ,
\end{equation*}
\begin{equation*}
    U^{T}(G-\mathbb{E}[G])U \leq    ( 1-v^{2}   )   \| G-\mathbb{E}[G]   \| ,
\end{equation*}
\begin{equation*}
    vU^{T}Y \leq |v| \sqrt{n(1-v^{2})} .
\end{equation*}
\begin{Lemma}
\label{BCBM-New Lemma}
Under the noisy label side information with noise parameter $\alpha$, 
\begin{equation*}
    \mathbb{P} \Bigg(\sum_{i=1}^{n} x_{i}^{*}y_{i} \leq \sqrt{n}\log n \Bigg) \leq e^{n\Big(\log \big(2\sqrt{\alpha(1-\alpha)}\big)+o(1) \Big)} .
\end{equation*}
\end{Lemma}
\begin{proof}
It follows from Chernoff bound. 
\end{proof}
Using Lemma~\ref{BCBM-New Lemma}, it can be shown that with probability converging to one, $\sum_{i=1}^{n} x_{i}^{*}y_{i} \geq \sqrt{n}\log n$.
Thus, 
\begin{equation*}
v^{2}  \bigg[ T_{2}\sqrt{n}\log n -2T_{2}\frac{\sqrt{n(1-v^{2})}}{|v|}-T_{1}p(1-2\xi)   \bigg] \geq 0 ,
\end{equation*}
as $n \rightarrow \infty$. Applying Lemma~\ref{BCBM-P-Lemma 2}, 
\begin{align}
\label{BCBM-N-equ5}
V^{T}S^{*}V & \geq  ( 1-v^{2}   )    \Big( \min_{i \in [n]} d_{i}^{*} - T_{1} c^{'}\sqrt{\log n} +T_{1}p(1-2\xi) \Big) .
\end{align}

\begin{Lemma}
\label{BCBM-N-Lemma 2}
Consider a sequence $f(n)$, and for each $n$ a sequence of i.i.d. random variables  $\{S_1,\ldots,S_m\}$ with distribution $p_{1} \delta_{+1} + p_{2} \delta_{-1} +(1-p_{1}-p_{2}) \delta_{0}$, where the parameters of the distribution depend on $n$ via $p_{1} = \rho_{1}\frac{\log n}{n}$, and $p_{2} = \rho_{2}\frac{\log n}{n}$ for some positive constants $\rho_{1}, \rho_{2}$. We assume $m(n)-n=o(n)$, where in the sequel the dependence of $m$ on $n$ is implicit. Define $\omega \triangleq \lim_{n\rightarrow\infty} \frac{f(n)}{\log n}$. 
For sufficiently large $n$, when $\omega<\rho_1-\rho_2 $,
\begin{equation}
 \mathbb{P} \Bigg(\sum_{i=1}^{m} S_{i} \leq f(n) \Bigg) \leq n^{-\eta^{*}+o(1)} ,
\label{eq:UpperUpper}\end{equation}
and when $\omega >\rho_1-\rho_2$,
\begin{equation}
\mathbb{P}\Bigg(\sum_{i=1}^{m} S_{i} \geq f(n) \Bigg) = n^{-\eta^{*}+o(1)}, 
\label{eq:LowerUpper}
\end{equation}
where $\eta^{*} = \rho_{1} + \rho_{2} -\gamma^{*} +\frac{\omega}{2} \log  \Big( \frac{\rho_{2} (\gamma^{*} + \omega)}{\rho_{1}(\gamma^{*} - \omega)}  \Big)$ and $\gamma^{*} = \sqrt{\omega^{2}+4\rho_{1}\rho_{2}}$.
\end{Lemma}
\begin{proof}
Inequality~\eqref{eq:UpperUpper} is derived by applying Chernoff bound. Equality~\eqref{eq:LowerUpper} is obtained by a sandwich argument on the probability: an upper bound derived via Chernoff bound, and a lower bound from~\cite[Lemma 15]{Ref7}. 
\end{proof}
% Since $ \frac{1}{n} \Big[T_{2}\sqrt{n} \frac{\sqrt{1-v^{2}}}{|v|}+T_{1}p(1-2\xi) \Big] = o(1)$, we have
It follows from~\eqref{BCBM-N-equ44} that
\begin{align*}
\mathbb{P} ( d_{i}^{*}\leq \delta  ) =& \mathbb{P}  \Bigg( \sum_{j=1}^{n} G_{ij}x_{i}^{*}x_{j}^{*} \leq \frac{\delta-T_{2}}{T_{1}}   \Bigg)  (1-\alpha  ) \\
&+ \mathbb{P}  \Bigg( \sum_{j=1}^{n} G_{ij}x_{i}^{*}x_{j}^{*} \leq \frac{\delta+T_{2}}{T_{1}}   \Bigg) \alpha , 
\end{align*}
where $\sum_{j=1}^{n} G_{ij}x_{i}^{*}x_{j}^{*}$ is equal in distribution to $\sum_{i=1}^{n-1} S_{i}$ in Lemma~\ref{BCBM-N-Lemma 2} with $p_{1} = p(1-\xi)$ and $p_{2}=p\xi$.

Recall that $\beta \triangleq \lim_{n \rightarrow \infty} \frac{T_{2}}{\log n}$, where $\beta \geq 0$.
First, we bound $\min_{i \in [n]}d_{i}^{*}$ under the condition $0\leq \beta<aT_{1}(1-2\xi)$. It follows from Lemma~\ref{BCBM-N-Lemma 2} that
\begin{align*}
&\mathbb{P}  \Bigg( \sum_{j=1}^{n} G_{ij}x_{i}^{*}x_{j}^{*} \leq \frac{\delta-T_{2}}{T_{1}}   \Bigg) \leq n^{-\eta(a,\beta)+o(1)}, \\
& \mathbb{P}  \Bigg( \sum_{j=1}^{n} G_{ij}x_{i}^{*}x_{j}^{*} \leq \frac{\delta+T_{2}}{T_{1}}   \Bigg) \leq n^{-\eta(a,\beta)+\beta+o(1)} . 
\end{align*}
Then
\begin{align*}
\mathbb{P} ( d_{i}^{*} \leq \delta  ) &\leq n^{-\eta(a,\beta)+o(1)} (1-\alpha ) + n^{-\eta(a,\beta)+\beta+o(1)} \alpha \\
& =n^{-\eta(a,\beta) +o(1)} .
\end{align*}
Using the union bound, 
\begin{equation*}
\mathbb{P} \bigg( \min_{i \in [n]}d_{i}^{*} \geq \frac{\log n}{\log \log n}   \bigg) \geq 1-n^{1-\eta(a, \beta)+o(1)} .
\end{equation*}
When $\eta(a,\beta)>1 $, it follows  $\min_{i \in [n]}d_{i}^{*} \geq \frac{\log n}{\log \log n}$ with probability $1-o(1)$. Thus, as long as $\eta(a,\beta)>1$, we can replace $\min d_i^*$  in~\eqref{BCBM-N-equ5} with $\frac{\log n}{\log \log n}$ and obtain, with probability $1-o(1)$:
\begin{align*}
V^{T}S^{*}V \geq&    ( 1-v^{2}  )   \bigg( \frac{\log n}{\log \log n} - T_{1} c' \sqrt{\log n}+T_{1}p(1-2\xi)    \bigg)\\
> &0 , 
\end{align*}
which concludes the first part of Theorem~\ref{Theorem 3}.

We now bound $\min_{i \in [n]}d_{i}^{*}$ under the condition $\beta>aT_{1}(1-2\xi)$. It follows from Lemma~\ref{BCBM-N-Lemma 2} that
% 			$\mathbb{P}  ( \sum_{j=1}^{n} G_{ij}x_{i}^{*}x_{j}^{*} \leq \frac{\delta-T_{2}}{T_{1}}   ) \leq n^{-\eta(a,\beta)+o(1)}$ and $\mathbb{P}  ( \sum_{j=1}^{n} G_{ij}x_{i}^{*}x_{j}^{*} \leq \frac{\delta+T_{2}}{T_{1}}   ) \leq 1 $. 
\begin{align*}
&\mathbb{P}  \Bigg( \sum_{j=1}^{n} G_{ij}x_{i}^{*}x_{j}^{*} \leq \frac{\delta-T_{2}}{T_{1}}   \Bigg) \leq n^{-\eta(a,\beta)+o(1)}, \\
&\mathbb{P}  \Bigg( \sum_{j=1}^{n} G_{ij}x_{i}^{*}x_{j}^{*} \leq \frac{\delta+T_{2}}{T_{1}}   \Bigg) \leq 1 .
\end{align*}
Then
\begin{equation*}
\mathbb{P} ( d_{i}^{*} \leq \delta  ) \leq n^{-\eta(a,\beta)+o(1)}+n^{-\beta+o(1)} , 
\end{equation*}
where $\alpha =  n^{-\beta +o(1)}$. 
Using the union bound,
\begin{equation*}
\mathbb{P} \Big( \min_{i \in [n]}d_{i}^{*} \geq \delta   \Big) \geq 1- \Big( n^{1-\eta(a,\beta)+o(1)}+n^{1-\beta+o(1)}  \Big) .
\end{equation*}

\begin{Lemma}
\label{BCBM-N-Lemma 3}
If $\beta > 1$, then $\eta(a,\beta) > 1$.
\end{Lemma}
\begin{proof}
Define $\psi(a,\beta) \triangleq \eta(a,\beta) -\beta$. It can be shown that $\psi(a,\beta)$ is a convex function in $\beta$. At the optimal $\beta^{*} $, $\log  \Big( \frac{(1-\xi) (\gamma^{*} + \beta^{*})}{\xi(\gamma^{*} - \beta^{*})}  \Big) = 2T_{1}$. Then
\begin{equation}
\label{BCBM-N-Lemma3-equ2}
\eta(a,\beta)  -\beta \geq a -\frac{\gamma^{*}}{T_{1}} .
\end{equation}
It can be shown that at the optimal $\beta^{*}$, 
\begin{equation*}
    \frac{\gamma^{*} +\beta^{*}}{\gamma^{*} - \beta^{*}} = \frac{1-\xi}{\xi} = \frac{4\xi(1-\xi)a^{2}T_{1}^{2}}{(\gamma^{*} - \beta^{*})^{2}} .
\end{equation*}
Then $\gamma^{*} = \beta^{*} +2\xi aT_{1}$ and~\eqref{BCBM-N-Lemma3-equ2} is written as
\begin{equation}
\label{BCBM-N-Lemma3-equ3}
\eta(a,\beta)  -\beta \geq a -2\xi a -\frac{\beta^{*}}{T_{1}} .
\end{equation}
Also, it can be shown that at $\beta^{*}$, $\gamma^{*} = \frac{\beta^{*}}{1-2\xi}$. This implies that $\beta^{*} =(1-2\xi)aT_{1}$. Substituting in~\eqref{BCBM-N-Lemma3-equ3} leads to $\eta(a,\beta)  - \beta \geq 0$, which implies that $\eta(a,\beta)  >1$ when $\beta >1$.
\end{proof}

When $\beta>1$, using Lemma~\ref{BCBM-N-Lemma 3}, it follows $\min_{i \in [n]}d_{i}^{*} \geq \frac{\log n}{\log \log n}$ with probability  $1-o(1)$. Substituting in~\eqref{BCBM-N-equ5}, if $\beta>1 $, with probability $1-o(1)$ we obtain: 
\begin{align*}
V^{T}S^{*}V \geq&    ( 1-v^{2}  )    \Big( \frac{\log n}{\log \log n} - T_{1} c' \sqrt{\log n}+T_{1}p(1-2\xi)    \Big) \\
>& 0 , 
\end{align*}
which concludes the second part of Theorem~\ref{Theorem 3}.

%%%%%%%%%%%%%%%%%%%%%%%%%%%%%%%%%%%%%%%%%%%%%%%%%%%%%%%%%%%%%%%%%%%%%%%%%%%%%%%%%%%%%%%%%%%%%%%%%%%%%%%%%
\section{Proof of Theorem~\ref{Theorem 4}}
\label{Proof-Theorem-4}
Since the prior distribution of $X^{*}$ is uniform, among all estimators, the maximum likelihood estimator minimizes the average error probability. Therefore, we only need to show that with high probability the maximum likelihood estimator fails. 
Let 
\begin{equation*}
F \triangleq \Bigg \{ \min_{i \in [n] }~  \bigg( T_{1}\sum_{j =1}^{n} G_{ij} x_{j}^{*} x_{i}^{*} + T_{2}x_{i}^{*}y_{i}   \bigg) \leq -T_{1} \Bigg \}. 
\end{equation*}
Then $\mathbb{P} ( \text{ML Fails}  ) \geq \mathbb{P} ( F  )$. Let $H$ denote the set of first $  \lfloor \frac{n}{\log^{2} n}   \rfloor$ nodes and $e (i, H  )$ denote the number of edges between node $i$ and nodes in the set $H \subset [n]$. It can be shown that
\begin{align*}
\min_{i \in [n] }~ &  \bigg( T_{1}\sum_{j \in [n]} G_{ij} x_{j}^{*} x_{i}^{*} + T_{2}x_{i}^{*}y_{i}   \bigg) \\
\leq & \min_{i \in H }~  \bigg( T_{1}\sum_{j \in [n]} G_{ij} x_{j}^{*} x_{i}^{*} + T_{2}x_{i}^{*}y_{i}   \bigg) \\
\leq & \min_{i \in H }~  \bigg( T_{1}\sum_{j \in H^{c}} G_{ij} x_{j}^{*} x_{i}^{*} + T_{2}x_{i}^{*}y_{i}   \bigg) + \max_{i \in H }~e (i, H  ) .
\end{align*}
%where $H^{c}$ is the complement set of $H$. 
Define
\begin{align*}
&E_{1} \triangleq \bigg \{\max_{i \in H }~e (i, H  ) \leq \delta -T_{1} \bigg\} , \\
&E_{2} \triangleq \bigg \{ \min_{i \in H }~  \bigg( T_{1}\sum_{j \in H^{c}} G_{ij} x_{j}^{*} x_{i}^{*} + T_{2}x_{i}^{*}y_{i}   \bigg) \leq -\delta \bigg\}.
\end{align*}
% $\max_{i \in H }~e (i, H  ) \leq \delta -T_{1}$ and  $\min_{i \in H }~  ( T_{1}\sum_{j \in H^{c}} G_{ij} x_{j}^{*} x_{i}^{*} + T_{2}x_{i}^{*}y_{i}   ) \leq -\delta $, 
Notice that $F \supset E_{1} \cap E_{2}$ and it suffices to show $\mathbb{P} ( E_{1}  )  \rightarrow 1$ and $\mathbb{P} ( E_{2}  )  \rightarrow 1$ to prove that the maximum likelihood estimator fails. Since $e (i, H  ) \sim \text{Binom}(  | H   |, a\frac{\log n }{n})$, from Lemma~\ref{BCBM-P-Lemma 5}, 
\begin{align*}
\mathbb{P} &  ( e (i, H  ) \geq \delta-T_{1}  ) \\
&\leq  \bigg( \frac{\log^{2} n}{ae \log \log n} - \frac{T_{1} \log n}{ae}  \bigg)^{T_{1}-\frac{\log n}{\log \log n}} e^{-\frac{a}{\log n}} \leq n^{-2+o(1)}.
\end{align*}
Using the union bound,  $\mathbb{P}  ( E_{1}  ) \geq 1- n^{-1+o(1)}$. 

Let
\begin{align*}
E &\triangleq \bigg\{ T_{1}\sum_{j \in H^{c}} G_{ij} x_{j}^{*} x_{i}^{*} + T_{2}x_{i}^{*}y_{i} \leq -\delta \bigg\}, \\
E_{\alpha} &\triangleq \bigg\{ \sum_{j \in H^{c}} G_{ij} x_{j}^{*} x_{i}^{*} \leq \frac{-\delta + T_{2}}{T_{1}} \bigg\}, \\
E_{1-\alpha} &\triangleq \bigg \{ \sum_{j \in H^{c}} G_{ij} x_{j}^{*} x_{i}^{*} \leq \frac{-\delta - T_{2}}{T_{1}} \bigg \}.
\end{align*}
Then
\begin{align}
\label{BCBM-N-equ6}
\mathbb{P}  ( E_{2}  ) & =1 - \prod_{i\in H}  [ 1- \mathbb{P}  ( E  )  ]  \overset {(a)}{=} 1 -  [ 1- \mathbb{P}  ( E  )   ]^{  | H   |}  \nonumber\\
& = 1 -  [ 1-\alpha \mathbb{P}  ( E_{\alpha}  ) -  (1-\alpha  ) \mathbb{P}  ( E_{1-\alpha}  )  ]^{  | H   |} ,
\end{align}
where $(a)$ holds because $  \{ T_{1}\sum_{j \in H^{c}} G_{ij} x_{j}^{*} x_{i}^{*} + T_{2}x_{i}^{*}y_{i}    \}_{i \in H}$ are mutually independent.	

First, we bound $\mathbb{P}(E_{2})$ under the condition $0 \leq \beta < aT_{1} (1-2\xi )$. Using Lemma~\ref{BCBM-N-Lemma 2}, $\mathbb{P}  ( E_{\alpha}  ) \geq n^{-\eta(a,\beta) + \beta +o(1)}$ and $\mathbb{P}  ( E_{1-\alpha}  ) \geq n^{-\eta(a,\beta) +o(1)}$. It follows from~\eqref{BCBM-N-equ6} that
\begin{align*}
\mathbb{P}  ( E_{2}  ) & \overset{(a)}{\geq} 1 -  \Big[ 1 - n^{-\eta(a,\beta)+o(1)}  \Big]^{  | H   |}  \\
& \overset{(b)}{\geq} 1 - \exp  \Big( -n^{1-\eta(a,\beta)+o(1)}  \Big) ,
\end{align*}
where $(a)$ holds because $\alpha =n^{-\beta + o(1)}$ and $(b)$ is due to $1+x \leq e^{x}$.
Therefore, if $\eta(a,\beta) <1$, then $\mathbb{P}  (E_{2} )  \rightarrow 1$ and the first part of Theorem~\ref{Theorem 4} follows.

We now bound $\mathbb{P}(E_{2})$ under the condition $\beta \geq aT_{1} (1-2\xi ) $. 
Reorganizing~\eqref{BCBM-N-equ6}, 
\begin{equation}
\label{BCBM-N-equ7}
\mathbb{P}  ( E_{2}  )  = 1 -  [ (1-\alpha  ) \mathbb{P}  ( E_{1-\alpha}^{c}) +\alpha \mathbb{P}  ( E_{\alpha}^{c}  )  ]^{  | H   |} ,
\end{equation}
where 
\begin{align*}
\mathbb{P}  ( E_{\alpha}^{c}  ) &= \mathbb{P}  \bigg( \sum_{j \in H^{c}} G_{ij} x_{j}^{*} x_{i}^{*} \geq \frac{-\delta + T_{2}}{T_{1}}  \bigg) , \\
\mathbb{P}  ( E_{1-\alpha}^{c}  )&= \mathbb{P}  \bigg( \sum_{j \in H^{c}} G_{ij} x_{j}^{*} x_{i}^{*} \geq \frac{-\delta - T_{2}}{T_{1}}  \bigg) .
\end{align*}
Also, $\sum_{j \in H^{c}} G_{ij}x_{i}^{*}x_{j}^{*}$ is equal in distribution to $\sum_{i=1}^{  | H^{c}   |-1} S_{i}$ in Lemma~\ref{BCBM-N-Lemma 2}, where $p_{1} = p(1-\xi)$ and $p_{2}=p\xi$.
Then $\mathbb{P}  ( E_{\alpha}^{c}  ) \leq n^{-\eta(a,\beta) + \beta +o(1) }$ and $\mathbb{P}  ( E_{1-\alpha}^{c}  ) \leq 1$. It follows from~\eqref{BCBM-N-equ7} that
\begin{align*}
\mathbb{P}  ( E_{2}  ) & \geq 1 -  \Big[(1-\alpha  ) +\alpha n^{-\eta(a,\beta) + \beta +o(1) } \Big]^{  | H   |} \\
& \overset{(a)}{=} 1 -  \Big[ 1 -n^{-\beta+o(1) } +n^{-\eta(a,\beta) +o(1) }   \Big]^{  | H   |} \\
& \overset{(b)}{\geq } 1 - e^{ -n^{1-\beta +o(1) }  \Big(1-n^{-\eta(a,\beta)+\beta+o(1)} \Big) },
\end{align*}
where $(a)$ holds because $\alpha = n^{-\beta + o(1)}$ and $(b)$ is due to $1+x<e^{x}$. Therefore, since $\beta \leq \eta(a, \beta)$, if $\beta <1$, then $\mathbb{P}  (E_{2} )  \rightarrow 1$ and the second part of Theorem~\ref{Theorem 4} follows.

%%%%%%%%%%%%%%%%%%%%%%%%%%%%%%%%%%%%%%%%%%%%%%%%%%%%%%%%%%%%%%%%%%%%%%%%%%%%%%%%%%%%%%%%%%%%%%%%%%%%%%%%%%%%%%%%%%%%%%
\section{Proof of Theorem~\ref{Theorem 5}}
\label{Proof-Theorem-5}
We begin by deriving sufficient conditions for the \ac{SDP} estimator to produce the true labels $X^*$.
\begin{Lemma}
\label{BCBM-G-Lemma 1}
The sufficient conditions of Lemma~\ref{BCBM-N-Lemma 1} apply to the general side information \ac{SDP}~\eqref{BCBM-G-equ3} by replacing $S_{A}^{*} = \tilde{Y}^{T}X^{*}$ and $S_B^*=-\tilde{Y}$.
\end{Lemma}
\begin{proof}
The proof is similar to the proof of Lemma~\ref{BCBM-N-Lemma 1}. 
% Then it can be shown that $Z^{*}$ is a unique primal optimal solution.
\end{proof}

It suffices to show that $S^{*}$, defined via its components $S_{A}^{*}$, $S_{B}^{*}$, and $S_{C}^{*}$, satisfies other conditions in Lemma~\ref{BCBM-G-Lemma 1} with probability $1-o(1)$. 
Let 
\begin{equation}
\label{BCBM-G-equ4}
d_{i}^{*}=T_{1} \sum_{j=1}^{n} G_{ij}x_{j}^{*}x_{i}^{*} + \tilde{y}_{i}x_{i}^{*} .
\end{equation}
Then $D^{*}X^{*} = T_{1}GX^{*}+\tilde{Y}$ and based on the definitions of $S_{A}^{*}$, $S_{B}^{*}$, and $S_{C}^{*}$ in  Lemma~\ref{BCBM-G-Lemma 1}, $S^{*}$ satisfies the condition $S^{*} [1, X^{*T}]^T =0$.
It remains to show that~\eqref{BCBM-N-equ1 New} holds, i.e., $S^{*}\succeq 0$ and $\lambda_{2}(S^{*})>0$ with probability  $1-o(1)$.
Let 
\begin{equation}
\label{BCBM-G-equ1 New}
    y_{max} \triangleq K \max_{k, m_{k}}  \bigg | \log   \bigg( \frac{\alpha_{+,m{k}}^{k}}{\alpha_{-,m_{k}}^{k}}   \bigg)   \bigg| ,
\end{equation}
where $k \in    \{ 1,2,\cdots,K   \}$ and $m_{k} \in    \{ 1,2,\cdots,M_{K}   \}$. For any $V$ such that $V^{T}[1, X^{*T}]^T=0$ and $  \| V   \|=1$, we have
\begin{align}
&V^{T}S^{*}V =v^{2} S_{A}^{*} -2vU^{T} \tilde{Y} +U^{T}D^{*}U - T_{1}U^{T}GU \nonumber\\
&\geq   ( 1-v^{2}   )  \bigg[\min_{i \in [n]} d_{i}^{*} - T_{1}  \| G-\mathbb{E}[G]   \| + T_{1} p(1-2\xi) \bigg] \nonumber\\
&+v^{2}  \bigg[ \tilde{Y}^{T}X^{*} -2y_{max}\frac{\sqrt{n(1-v^{2})}}{|v|}-T_{1}p(1-2\xi)   \bigg] ,
\end{align}
where the last inequality holds in a manner similar to \eqref{BCBM-N-equ4} with the difference that in the present case
\begin{equation*}
vU^{T}\tilde{Y} \leq |v|y_{max} \sqrt{n(1-v^{2})}.
\end{equation*}
\begin{Lemma}
\label{BCBM-G-New Lemma}
For feature $k$ of general side information,
\begin{equation*}
    \mathbb{P} \Bigg(\sum_{i=1}^{n} x_{i}^{*}z_{i, k} \geq \sqrt{n}\log n \Bigg) \geq 1-o(1) ,
\end{equation*}
where 
\begin{align*}
    z_{i,k} &\triangleq \sum_{m_{k}=1}^{M_{k}} \mathbbm{1}_{\{y_{i,k} = m_{k}\} } \log\bigg( \frac{\alpha_{+,m_{k}}^{k}}{\alpha_{-,m_{k}}^{k}}\bigg).
\end{align*}
\end{Lemma}
\begin{proof}
For feature $k$, let
\begin{align*}
    \delta' &\triangleq \sqrt{n} \log n , \\
    % A_{j} &\triangleq \{i  \in [n] : y_{i,k} = j \}, \\
    \rho_{j} &\triangleq \frac{1}{n} |\{i  \in [n] : y_{i,k} = j \}|,
\end{align*}
where $j \in \{1, \cdots, M_{k}\}$ and $\sum_{j} \rho_{j} = 1$. Then
\begin{align*}
    \mathbb{P} \Bigg(\sum_{i=1}^{n} x_{i}^{*}z_{i,k} \leq \delta' \Bigg) &\leq \sum_{j=1}^{M_{k}} \mathbb{P} \Bigg( \sum_{i\in A_{j}} x_{i}^{*}z_{i,k} \leq \delta' \Bigg).
\end{align*}
Applying Chernoff bound yields
\begin{equation*}
    \mathbb{P} \Bigg( \sum_{i\in A_{j}} x_{i}^{*}z_{i,k} \leq \delta' \Bigg) \leq e^{n(\psi_{k,j}+o(1))}, 
\end{equation*}
where 
\begin{equation*}
    \psi_{k,j} \triangleq \rho_{j} \log \bigg( 2\sqrt{\alpha_{+,j}^{k} \alpha_{-,j}^{k} \mathbb{P}(x_{i}^{*}=1) \mathbb{P}(x_{i}^{*}=-1)} \bigg ).
\end{equation*}
Since $\psi_{k,j}<0$ for any values of $\alpha_{+,j}^{k}$ and $\alpha_{-,j}^{k}$, we have 
\begin{align*}
    \mathbb{P} \Bigg(\sum_{i=1}^{n} x_{i}^{*}z_{i,k} \geq \delta' \Bigg) &\geq 1- \sum_{j=1}^{M_{k}} e^{n(\psi_{k,j}+o(1))} = 1 - o(1).
\end{align*}
Therefore, with probability $1-o(1)$,  $\sum_{i=1}^{n} x_{i}^{*}z_{i,k} \geq \sqrt{n} \log n$ and Lemma~\ref{BCBM-G-New Lemma} follows.
\end{proof}

Using Lemmas~\ref{BCBM-P-Lemma 2} and~\ref{BCBM-G-New Lemma}, 
\begin{align}
\label{BCBM-G-equ5}
V^{T}S^{*}V \geq&  ( 1-v^{2}   )    \Big( \min_{i \in [n]}d_{i}^{*} - T_{1} c^{'}\sqrt{\log n} +T_{1}p(1-2\xi)    \Big) .
\end{align}
It can be shown that $\sum_{j=1}^{n} G_{ij}x_{i}^{*}x_{j}^{*}$ in~\eqref{BCBM-G-equ4} is equal in distribution to $\sum_{i=1}^{n-1} S_{i}$ in Lemma~\ref{BCBM-N-Lemma 2}, where $p_{1} = p(1-\xi)$ and $p_{2}=p\xi$. Then 
\begin{equation}
\label{BCBM-G-equ6}
\mathbb{P} ( d_{i}^{*}\leq \delta  ) = \sum_{m_{1}=1}^{M_{1}} \sum_{m_{2}=1}^{M_{2}} ... \sum_{m_{K}=1}^{M_{K}} P (m_{1}, ..., m_{K} ) , 
\end{equation}
where 
\begin{align*}
P &  (m_{1}, ..., m_{K} )  \triangleq  \mathbb{P} (x_{i}=1 )e^{f_{2}(n)} \mathbb{P}   \Bigg( \sum_{i=1}^{n-1} S_{i} \leq \frac{\delta-f_{1}(n)}{T_{1}}   \Bigg) \\
&+\mathbb{P} (x_{i}=-1 ) e^{f_{3}(n)} \mathbb{P}   \Bigg( \sum_{i=1}^{n-1} S_{i} \leq \frac{\delta+f_{1}(n)}{T_{1}}   \Bigg).
\end{align*}

% Note that $f_{1}(n) = f_{2}(n)-f_{3}(n)$. 
First, we bound $\min_{i \in [n]}d_{i}^{*}$ under the condition $  | \beta_{1}   | \leq aT_{1} (1-2\xi )$. It follows from Lemma~\ref{BCBM-N-Lemma 2} that 
\begin{align*}
\mathbb{P}  \Bigg( \sum_{i=1}^{n-1} S_{i} \leq \frac{\delta-f_{1}(n)}{T_{1}}   \Bigg) &\leq n^{-\eta(a,\beta_{1})+o(1)}, \\
\mathbb{P}  \Bigg( \sum_{i=1}^{n-1} S_{i} \leq \frac{\delta+f_{1}(n)}{T_{1}}   \Bigg) &\leq n^{-\eta(a,\beta_{1})-\beta_{1}+o(1)}.
\end{align*}
Notice that 
\begin{align*}
\beta \triangleq 
\lim_{n\rightarrow\infty} -\frac{\max (f_2(n),f_3(n))}{\log n} .
\end{align*}
When $\beta_{1} \geq 0$, $\lim_{n\rightarrow \infty}\frac{f_{2}(n)}{\log n} = -\beta$ and $\lim_{n\rightarrow \infty}\frac{f_{3}(n)}{\log n} = -\beta_{1}-\beta$. Then
\begin{equation*}
\mathbb{P} ( d_{i}^{*} \leq \delta  ) \leq n^{-\eta(a,\beta_{1})-\beta+o(1)}. 
\end{equation*}
When $ \beta_{1} < 0$, $\lim_{n\rightarrow \infty}\frac{f_{3}(n)}{\log n} = -\beta$ and $\lim_{n\rightarrow \infty}\frac{f_{2}(n)}{\log n} = \beta_{1} - \beta$. 
Then
\begin{equation*}
\mathbb{P} ( d_{i}^{*} \leq \delta  ) \leq n^{-\eta(a,\beta_{1})+\beta_{1}-\beta+o(1)} = n^{-\eta(a,  | \beta_{1}   |)-\beta+o(1)}. 
\end{equation*}
Using the union bound, 
\begin{equation*}
\mathbb{P} \bigg( \min_{i \in [n]}d_{i}^{*} \geq \frac{\log n}{\log \log n}   \bigg) \geq 1-n^{1-\eta(a,  | \beta_{1}   |) -\beta+o(1)} .
\end{equation*}
When $\eta(a,  | \beta_{1}   |) +\beta >1$, it follows that $\min_{i \in [n]}d_{i}^{*} \geq \frac{\log n}{\log \log n}$ holds with probability $1-o(1)$. Substituting into~\eqref{BCBM-G-equ5}, if $\eta(a,  | \beta_{1}   |) +\beta >1 $, then with probability $1-o(1)$,
\begin{align*}
V^{T}S^{*}V \geq&    ( 1-v^{2}  )  \bigg( \frac{\log n}{\log \log n} - T_{1} c' \sqrt{\log n}+T_{1}p(1-2\xi)    \bigg) \\
> & 0 , 
\end{align*}
which concludes the first part of Theorem~\ref{Theorem 5}.

We now bound $\min_{i \in [n]}d_{i}^{*}$ under the condition $  | \beta_{1}   | \geq aT_{1} (1-2\xi )$. When $\beta_{1} \geq 0$, $\lim_{n\rightarrow \infty}\frac{f_{2}(n)}{\log n} = -\beta$ and $\lim_{n\rightarrow \infty}\frac{f_{3}(n)}{\log n} = -\beta_{1}-\beta$. Then
\begin{equation*}
\mathbb{P} ( d_{i}^{*} \leq \delta  ) \leq n^{-\beta+o(1)} + n^{-\beta-\beta_{1}+o(1)}. 
\end{equation*}
When $\beta_{1} < 0$, $\lim_{n\rightarrow \infty} \frac{f_{3}(n)}{\log n} = -\beta$ and $\lim_{n\rightarrow \infty} \frac{f_{2}(n)}{\log n} = \beta_{1}-\beta$. Then
\begin{equation*}
\mathbb{P} ( d_{i}^{*} \leq \delta  ) \leq n^{-\beta+\beta_{1}+o(1)} + n^{-\beta+o(1)}. 
\end{equation*}
Using the union bound,
\begin{equation*}
\mathbb{P} \bigg( \min_{i \in [n]}d_{i}^{*} \geq \frac{\log n}{\log \log n}   \bigg) \geq 1-n^{1-  | \beta_{1}   | -\beta+o(1)} .
\end{equation*}
When $  | \beta_{1}   | +\beta >1 $, with probability $1-o(1)$, we have $\min_{i \in [n]}d_{i}^{*} \geq \frac{\log n}{\log \log n}$. Substituting into~\eqref{BCBM-G-equ5}, if $  | \beta_{1}   | +\beta >1 $, then with probability $1-o(1)$,
\begin{align*}
V^{T}S^{*}V \geq&    ( 1-v^{2}  )    \bigg( \frac{\log n}{\log \log n} - T_{1} c' \sqrt{\log n}+T_{1}p(1-2\xi)    \bigg)\\
> &0 , 
\end{align*}
which concludes the second part of Theorem~\ref{Theorem 5}.

%%%%%%%%%%%%%%%%%%%%%%%%%%%%%%%%%%%%%%%%%%%%%%%%%%%%%%%%%%%%%%%%%%%%%%%%%%%%%%%%%%%%%%%%%%%%%%%%%%%%%%%%%%%%%%%%
\section{Proof of Theorem~\ref{Theorem 6}}
\label{Proof-Theorem-6}
Similar to the proof of Theorem~\ref{Theorem 4}, let 
\begin{equation*}
F \triangleq \Bigg\{ \min_{i \in [n] }~  \Bigg( T_{1}\sum_{j =1}^{n} G_{ij} x_{j}^{*} x_{i}^{*} + x_{i}^{*}\tilde{y}_{i}   \Bigg) \leq -T_{1} \Bigg\}.
\end{equation*}
Then $\mathbb{P} ( \text{ML Fails}  ) \geq \mathbb{P} ( F  )$ and if we show that $\mathbb{P} ( \text{F}  )  \rightarrow 1$, the maximum likelihood estimator fails. 
Let $H$ be the set of first $  \lfloor \frac{n}{\log^{2} n}   \rfloor$ nodes and $e (i, H  )$ denote the number of edges between node $i$ and other nodes in the set $H$. It can be shown that
\begin{align*}
\min_{i \in [n] }~ &  \Bigg( T_{1}\sum_{j \in [n]} G_{ij} x_{j}^{*} x_{i}^{*} + x_{i}^{*}\tilde{y}_{i}   \Bigg) \\
\leq & \min_{i \in H }~  \Bigg( T_{1}\sum_{j \in [n]} G_{ij} x_{j}^{*} x_{i}^{*} + x_{i}^{*}\tilde{y}_{i}   \Bigg) \\
\leq & \min_{i \in H }~  \Bigg( T_{1}\sum_{j \in H^{c}} G_{ij} x_{j}^{*} x_{i}^{*} + x_{i}^{*}\tilde{y}_{i}   \Bigg) + \max_{i \in H }~e (i, H  ) .
\end{align*}
%where $H^{c}$ is the complement set of $H$. 
Let
\begin{align*}
&E_{1} \triangleq \Bigg\{ \max_{i \in H }~e (i, H  ) \leq \delta -T_{1} \Bigg\}, \\
&E_{2} \triangleq \Bigg\{ \min_{i \in H }~  \Bigg( T_{1}\sum_{j \in H^{c}} G_{ij} x_{j}^{*} x_{i}^{*} + x_{i}^{*}\tilde{y}_{i}   \Bigg) \leq -\delta \Bigg\}.
\end{align*}
% $\max_{i \in H }~e (i, H  ) \leq \delta -T_{1}$ and  $\min_{i \in H }~  ( T_{1}\sum_{j \in H^{c}} G_{ij} x_{j}^{*} x_{i}^{*} + x_{i}^{*}\tilde{y}_{i}   ) \leq -\delta $, 
Notice that $F \supset E_{1} \cap E_{2}$. Then the maximum likelihood estimator fails if we show that $\mathbb{P} ( E_{1}  )  \rightarrow 1$ and $\mathbb{P} ( E_{2}  )  \rightarrow 1$. Since $e (i, H  ) \sim \text{Binom}(  | H   |, a\frac{\log n }{n})$, from Lemma~\ref{BCBM-P-Lemma 5}, 
\begin{align*}
\mathbb{P} &  ( e (i, H  ) \geq \delta-T_{1}  ) \\
&\leq  \bigg( \frac{\log^{2} n}{ae \log \log n} - \frac{T_{1} \log n}{ae}  \bigg)^{T_{1}-\frac{\log n}{\log \log n}} e^{-\frac{a}{\log n}} \leq  n^{-2+o(1)} .
\end{align*}
Using the union bound,  $\mathbb{P}  ( E_{1}  ) \geq 1- n^{-1+o(1)}$. 

Let
\begin{align*}
E &\triangleq \Bigg\{ T_{1}\sum_{j \in H^{c}} G_{ij} x_{j}^{*} x_{i}^{*} + x_{i}^{*}\tilde{y}_{i} \leq -\delta \Bigg\}, \\
E_{+} &\triangleq \Bigg\{ \sum_{j \in H^{c}} G_{ij} x_{j}^{*} x_{i}^{*} \leq \frac{-\delta - f_{1}(n) }{T_{1}} \Bigg\}, \\
E_{-} &\triangleq \Bigg\{ \sum_{j \in H^{c}} G_{ij} x_{j}^{*} x_{i}^{*} \leq \frac{-\delta + f_{1}(n)}{T_{1}} \Bigg\}.
\end{align*}
% $T_{1}\sum_{j \in H^{c}} G_{ij} x_{j}^{*} x_{i}^{*} + x_{i}^{*}\tilde{y}_{i} \leq -\delta$, $\sum_{j \in H^{c}} G_{ij} x_{j}^{*} x_{i}^{*} \leq \frac{-\delta - f_{1}(n) }{T_{1}}$, and $\sum_{j \in H^{c}} G_{ij} x_{j}^{*} x_{i}^{*} \leq \frac{-\delta + f_{1}(n)}{T_{1}}$, 
Define
\begin{align*}
P(m_{1}, ..., m_{K}) \triangleq &  \mathbb{P} (x_{i}^{*}=1 ) e^{f_{2}(n)} \mathbb{P}   ( E_{+}   )  \nonumber\\
&+\mathbb{P} (x_{i}^{*}=-1 ) e^{f_{3}(n)} \mathbb{P}   ( E_{-}   ).
\end{align*}
Then
\begin{align*}
\mathbb{P}  ( E_{2}  ) & =1 - \prod_{i\in H}  [ 1- \mathbb{P}  ( E  )  ] \overset {(a)}{=} 1 -  [ 1- \mathbb{P}  ( E  )   ]^{  | H   |}  \\
& = 1 -  \Bigg[ 1- \sum_{m_{1}=1}^{M_{1}} \cdots \sum_{m_{K}=1}^{M_{K}} P (m_{1}, ..., m_{K} )  \Bigg]^{  | H   |} ,
\end{align*}
where $(a)$ holds because $  \{ T_{1}\sum_{j \in H^{c}} G_{ij} x_{j}^{*} x_{i}^{*} + x_{i}^{*}\tilde{y}_{i}    \}_{i \in H}$ are mutually independent.

\label{Remark2}
First, we bound $\mathbb{P}(E_2)$ under the condition $| \beta_{1} | \leq aT_{1} (1-2\xi )$. Using Lemma~\ref{BCBM-N-Lemma 2}, $\mathbb{P}  ( E_{+}   )  \geq n^{-\eta(a,\beta_{1})+o(1)}$ and $\mathbb{P}  ( E_{-}   )  \geq n^{-\eta(a,\beta_{1})+\beta_{1}+o(1)}$.
When $\beta_{1}\geq 0$, $\lim_{n\rightarrow \infty} \frac{f_{2}(n)}{\log n}= -\beta$ and $\lim_{n\rightarrow \infty} \frac{f_{3}(n)}{\log n}= -\beta_{1} -\beta$. Then
\begin{align*}
\mathbb{P}  ( E_{2}  ) & = 1 -  \Big[ 1- n^{-\eta(a,\beta_{1})-\beta+o(1)}   \Big]^{  | H   |}  \\
& \geq 1 - \exp  \Big( -n^{1-\eta(a,\beta_{1})-\beta+o(1)}  \Big) ,
\end{align*}
using $1+x \leq e^{x}$.
When $\beta_{1}<0$, $\lim_{n\rightarrow \infty} \frac{f_{3}(n)}{\log n} = -\beta$ and $\lim_{n\rightarrow \infty} \frac{f_{2}(n)}{\log n} = \beta_{1} -\beta $. Then
\begin{align*}
\mathbb{P}  ( E_{2}  ) & = 1 -  \Big[ 1- n^{-\eta(a,\beta_{1})+\beta_{1}-\beta+o(1)}  \Big]^{  | H   |}  \\
& \geq 1 - \exp  \Big( -n^{1-\eta(a,  | \beta_{1}   |)-\beta+o(1)}  \Big) ,
\end{align*}
using $1+x \leq e^{x}$  and $\eta(a,\beta_{1})-\beta_{1} = \eta(a,  | \beta_{1}   |)$.
Therefore, if $\eta(a,  | \beta_{1}   |) +\beta<1$, then $\mathbb{P}  (E_{2} )  \rightarrow 1$ and the first part of Theorem~\ref{Theorem 6} follows.

We now bound $\mathbb{P}(E_2)$ under the condition $  | \beta_{1}   | \geq aT_{1} (1-2\xi )$. When $\beta_{1} \geq 0$, $\lim_{n\rightarrow \infty} \frac{f_{2}(n)}{\log n} = -\beta$ and $\lim_{n\rightarrow \infty} \frac{f_{3}(n)}{\log n} = -\beta_{1} -\beta$. Using Lemma~\ref{BCBM-N-Lemma 2}, $\mathbb{P}  ( E_{+}   )  \geq n^{-\eta(a, \beta_{1})+o(1)}$ and $\mathbb{P}  ( E_{-}   )  \geq 1-o(1)$. Then
\begin{align*}
\mathbb{P} & ( E_{2}  ) \geq 1 -  \Big[ 1 -n^{-\eta(a,\beta_{1})-\beta+o(1)} -n^{-\beta_{1}-\beta+o(1)}  \Big]^{  | H   |}  \\
& \geq 1 - \exp  \Big( -n^{1-\eta(a,\beta_{1})-\beta+o(1)} -n^{1-\beta_{1}-\beta+o(1)}  \Big) ,
\end{align*}
using $1+x \leq e^{x}$.
When $\beta_{1} < 0$, $\lim_{n\rightarrow \infty} \frac{f_{3}(n)}{\log n} = -\beta$ and $\lim_{n\rightarrow \infty} \frac{f_{2}(n)}{\log n}= \beta_{1} -\beta$. Using Lemma~\ref{BCBM-N-Lemma 2}, $\mathbb{P}  ( E_{+}   )  \geq 1-o(1)$ and $\mathbb{P}  ( E_{-}   )  \geq n^{-\eta(a, |\beta_{1}|)+o(1)}$. Then
\begin{align*}
\mathbb{P} & ( E_{2}  ) = 1 -  \Big[ 1 -n^{\beta_{1}-\beta+o(1)} -n^{-\eta(a,|\beta_{1}|)-\beta +o(1)}  \Big]^{  | H   |}  \\
& \geq 1 - \exp  \Big( -n^{1+\beta_{1}-\beta+o(1)} -n^{1-\eta(a,|\beta_{1}|)-\beta +o(1)}  \Big) ,
\end{align*}
using $1+x \leq e^{x}$.
Therefore, since $|\beta_{1}| \leq \eta(a,|\beta_{1}|)$, if $|\beta_{1}| +\beta<1$, then $\mathbb{P}  (E_{2} )  \rightarrow 1$ and the second part of Theorem~\ref{Theorem 6} follows.

%%%%%%%%%%%%%%%%%%%%%%%%%%%%%%%%%%%%%%%%%%%%%%%%%%%%%%%%%%%%%%%%%%%%%%%%%%%%%%%%%%%%%%%%%%%%%%%%%%%%%%%%%%%%%%%%%%%%
\section{Proof of Theorem~\ref{Theorem 7}}
\label{Proof-Theorem-7}
We begin by deriving sufficient conditions for the solution of \ac{SDP}~\eqref{BSSBM-P-equ2} to match the true labels.
\begin{Lemma}
\label{BSSBM-P-Lemma 1}
For the optimization problem~\eqref{BSSBM-P-equ2}, consider the Lagrange multipliers
\begin{equation*}
\lambda^{*} \quad , \quad \mu^* \quad , \quad D^{*}=\mathrm{diag}(d_{i}^{*}), \quad 
S^{*}.
\end{equation*}
If we have
\begin{align*}
&S^{*} = D^{*}+\lambda^{*} \mathbf{J}+\mu^{*}W-G ,\\
& S^{*} \succeq 0, \\
&\lambda_{2}(S^{*}) >  0 ,\\
&S^{*}X^{*} =0 ,
\end{align*}
then $(\lambda^{*}, \mu^{*}, D^*, S^*)$ is the dual optimal solution and $\Zsdp=X^{*}X^{*T}$ is the unique primal optimal solution of~\eqref{BSSBM-P-equ2}.
\end{Lemma}

\begin{proof}
The proof is similar to the proof of Lemma~\ref{BCBM-P-Lemma 1}. The Lagrangian of~\eqref{BSSBM-P-equ2} is given by
\begin{align*}
L(Z,S,D,\lambda,\mu)=&\langle G,Z\rangle +\langle S,Z\rangle -\langle D,Z-\mathbf{I} \rangle \\
&-\lambda \langle \mathbf{J},Z\rangle -\mu  ( \langle W,Z\rangle -(Y^{T}Y)^{2}  ) ,
\end{align*}
where $S\succeq 0$, $D=\mathrm{diag}(d_{i})$, $\lambda ,\mu$ are Lagrange multipliers. Since $\langle \mathbf{J},Z \rangle = 0$, for any $Z$ that satisfies the constraints in~\eqref{BSSBM-P-equ2}, it can be shown that $\langle G,Z\rangle \leq \langle G,Z^{*}\rangle$. Also, similar to the proof of Lemma~\ref{BCBM-P-Lemma 1}, it can be shown that the optimum solution is unique.
\end{proof}

It suffices to show that $S^{*} = D^{*}+\lambda^{*} \mathbf{J}+\mu^{*}W-G$ satisfies other conditions in Lemma~\ref{BSSBM-P-Lemma 1} with probability $1-o(1)$.
% We now show that $S^{*} = D^{*}+\mu^{*}W-G$ satisfies other conditions in Lemma~\ref{BCBM-N-Lemma 1} with probability $1-o(1)$. 
Let 
\begin{equation}
\label{BSSBM-P-equ3}
d_{i}^{*}= \sum_{j=1}^{n} G_{ij}x_{j}^{*}x_{i}^{*} -\mu^{*} \sum_{j=1}^{n} y_{i}y_{j}x_{j}^{*}x_{i}^{*} .
\end{equation}
Then $D^{*}X^{*} = GX^{*}-\mu^{*}WX^{*}$ and based on the definition of $S^{*}$ in  Lemma~\ref{BSSBM-P-Lemma 1}, $S^{*}$ satisfies the condition $S^{*}X^{*} =0$.
It remains to show that~\eqref{BCBM-P-equ1-New} holds, i.e., $S^{*}\succeq 0$ and $\lambda_{2}(S^{*})>0$ with probability $1-o(1)$. 
Under the binary stochastic block model,
\begin{equation}
    \label{BSBM-expectation}
    \mathbb{E}[G]=\frac{p-q}{2}X^{*}X^{*T}+\frac{p+q}{2}\mathbf{J}-p\mathbf{I} .
\end{equation}
It follows that for any $V$ such that $V^{T}X^{*}=0$ and $  \| V   \|=1$,
\begin{align*}
V^{T}S^{*}V=&V^{T}D^{*}V+  ( \lambda^{*} -\frac{p+q}{2}   )V^{*}\mathbf{J}V+p\\
&-V^{T}(G-\mathbb E[G])V+\mu^{*}V^{T}WV .
\end{align*}
Let $\lambda^{*}\geq \frac{p+q}{2}$. Since $V^{T}D^{*}V \geq \min_{i\in [n]} d_{i}^{*}$ and $V^{T}(G-\mathbb E[G])V \leq    \| G-\mathbb E[G]   \| $, 
\begin{equation*}
V^{T}S^{*}V \geq \min_{i \in [n]} d_{i}^{*}+p-  \| G-\mathbb E[G]   \|+\mu^{*}V^{T}WV .
\end{equation*}

\begin{Lemma}\cite[Thoerem 5]{Ref18}
\label{BSSBM-P-Lemma 2}
%		Let $G$ be the adjacency matrix of a graph under the binary symmetric stochastic block model. 
For any $c > 0$, there exists $c^{'} >0$ such that for any $n \geq 1$, $  \| G-\mathbb E[G]   \| \leq c^{'}\sqrt{\log n}$ with probability at least $1-n^{-c}$. 
\end{Lemma}

Also, it can be shown that Lemma~\ref{BCBM-P-Lemma 3} holds here. Choose $\mu^{*}< 0 $, then in view of Lemmas~\ref{BSSBM-P-Lemma 2} and~\ref{BCBM-P-Lemma 3}, with probability $1-o(1)$,
\begin{equation}
\label{BSSBM-P-equ4}
V^{T}S^{*}V \geq  \min_{i \in [n]}d_{i}^{*} +p+(\mu^{*}-c^{'})\sqrt{\log n}.
\end{equation}

\begin{Lemma} 
\label{BSSBM-P-Lemma 3}
When $\delta = \frac{\log n}{\log log n}$, then
\begin{equation*}
\mathbb{P}(d_{i}^{*}\leq \delta ) \leq \epsilon n^{-\frac{1}{2}  ( \sqrt{a}-\sqrt{b}   )^{2}+o(1)} + (1-\epsilon )\epsilon^{n}.
\end{equation*}
\end{Lemma}
\begin{proof}
It follows from Chernoff bound.
\end{proof}

Recall that $\beta \triangleq \lim_{n \rightarrow \infty} -\frac{\log \epsilon}{\log n}$, where $\beta \geq 0$. It follows from Lemma~\ref{BSSBM-P-Lemma 3} that 
% $\mathbb{P}(d_{i}^{*}\leq \delta ) \leq n^{-\frac{1}{2}  ( \sqrt{a}-\sqrt{b}   )^{2} - \beta +o(1)}$.
\begin{equation*}
\mathbb{P}(d_{i}^{*}\leq \delta ) \leq n^{-\frac{1}{2}  ( \sqrt{a}-\sqrt{b}   )^{2} - \beta +o(1)}.
\end{equation*}
Then using the union bound,
\begin{equation*}
\mathbb{P}\bigg(\min_{i \in [n]}d_{i}^{*} \geq \frac{\log n}{\log \log n} \bigg)   \geq 1 - n^{1-\frac{1}{2}  ( \sqrt{a}-\sqrt{b}   )^{2} - \beta +o(1)} .
\end{equation*}
When $ ( \sqrt{a}-\sqrt{b} )^{2} + 2\beta> 2$, it follows that $\min_{i \in [n]}d_{i}^{*} \geq \frac{\log n}{\log \log n}$ holds with probability $1-o(1)$. Combining this result with~\eqref{BSSBM-P-equ4}, if $ ( \sqrt{a}-\sqrt{b} )^{2} + 2\beta> 2$, then with probability $1-o(1)$,
\begin{equation*}
V^{T}S^{*}V  \geq  \frac{\log n}{\log \log n}+p+(\mu^{*}-c^{'})\sqrt{\log n} > 0 ,
\end{equation*}
which completes the proof of Theorem \ref{Theorem 7}. 

%%%%%%%%%%%%%%%%%%%%%%%%%%%%%%%%%%%%%%%%%%%%%%%%%%%%%%%%%%%%%%%%%%%%%%%%%%%%%%%%%%%%%%%%%%%%%%%%%%%%%%%%%%%%%%%%%%%%
\section{Proof of Theorem~\ref{Theorem 9}}
\label{Proof-Theorem-9}
We begin by deriving sufficient conditions for the solution of \ac{SDP}~\eqref{BSSBM-N-equ1} to match the true labels.
\begin{Lemma}
\label{BSSBM-N-Lemma 1}
For the optimization problem~\eqref{BSSBM-N-equ1}, consider the Lagrange multipliers
\begin{equation*}
\lambda^* \quad, \quad D^{*}=\mathrm{diag}(d_{i}^{*}), \quad 
S^{*}\triangleq \begin{bmatrix} S_{A}^{*} & S_{B}^{*T} \\ S_{B}^{*} & S_{C}^{*} \end{bmatrix}.
\end{equation*}
If we have
\begin{align*}
&S_{A}^{*} = T_{2}Y^{T}X^{*} ,\\
&S_{B}^{*} = -T_{2}Y ,\\
&S_{C}^{*} = D^{*}+\lambda^*{\mathbf J}-T_{1}G ,\\
& S^{*} \succeq 0, \\
&\lambda_{2}(S^{*}) >  0 ,\\
&S^{*} [1, X^{*T}]^T =0 ,
\end{align*}
then $(\lambda^{*}, D^*, S^*)$ is the dual optimal solution and $\Zsdp=X^{*}X^{*T}$ is the unique primal optimal solution of~\eqref{BSSBM-N-equ1}.
\end{Lemma}

\begin{proof}
The proof is similar to the proof of Lemma~\ref{BCBM-N-Lemma 1}. The Lagrangian of~\eqref{BSSBM-N-equ1} is given by
\begin{align*}
L(Z,X,S,D,\lambda)=&T_{1} \langle G,Z \rangle +T_{2} \langle Y, X \rangle +\langle S, H \rangle \\
&-\langle D,Z-\mathbf{I} \rangle -\lambda \langle \mathbf{J},Z \rangle ,
\end{align*}
where $S\succeq 0$, $D=\mathrm{diag}(d_{i})$, and $\lambda \in \mathbb{R}$ are Lagrange multipliers. Since $\langle \mathbf{J},Z^{*} \rangle = 0$, for any $Z$ that satisfies the constraints in~\eqref{BSSBM-N-equ1}, it can be shown that $T_{1}\langle G, Z\rangle + T_{2}\langle Y , X \rangle \leq T_{1}\langle G, Z^{*}\rangle + T_{2}\langle Y , X^{*} \rangle$. Also, the uniqueness of optimum solution is proved similarly.
\end{proof}

We now show that $S^{*}$ defined by $S_{A}^{*}$, $S_{B}^{*}$, and $S_{C}^{*}$ satisfies the remaining conditions in Lemma~\ref{BSSBM-N-Lemma 1} with probability $1-o(1)$. 
Let 
\begin{equation}
\label{BSBM-N-equ101}
d_{i}^{*}=T_{1} \sum_{j=1}^{n} G_{ij}x_{j}^{*}x_{i}^{*} + T_{2}y_{i}x_{i}^{*}.
\end{equation}
Then $D^{*}X^{*} = T_{1}GX^{*}+T_{2}Y$ and based on the definitions of $S_{A}^{*}$, $S_{B}^{*}$, and $S_{C}^{*}$ in  Lemma~\ref{BSSBM-N-Lemma 1}, $S^{*}$ satisfies the condition $S^{*} [1, X^{*T}]^T =0$.
It remains to show that~\eqref{BCBM-N-equ1 New} holds, i.e., $S^{*}\succeq 0$ and $\lambda_{2}(S^{*})>0$ with probability $1-o(1)$.
 
For any $V$ such that $V^{T}[1, X^{*T}]^T=0$ and $  \| V   \|=1$, we have
\begin{align}
\label{BSBM-N-equ102}
V^{T}&S^{*}V=v^{2} S_{A}^{*} -2vT_{2}U^{T}Y +U^{T}D^{*}U - T_{1}U^{T}GU  \nonumber\\
\geq &   ( 1-v^{2}   )  \bigg[\min_{i \in [n]}d_{i}^{*} - T_{1}  \| G-\mathbb{E}[G]   \| + T_{1} p \bigg] \nonumber \\
&+v^{2}  \bigg[ Y^{T}X^{*} -2T_{2} \frac{\sqrt{n(1-v^{2})}}{|v|} -T_{1}\frac{p-q}{2}    \bigg] ,
\end{align}
where the last inequality holds in a manner similar to \eqref{BCBM-N-equ4}.
Using Lemma~\ref{BCBM-New Lemma}, 
\begin{equation}
\label{BSSBM-N-equ3}
V^{T}S^{*}V \geq  ( 1-v^{2}   )   \bigg( \min_{i \in [n]} d_{i}^{*} - T_{1} c^{'}\sqrt{\log n} +T_{1}p   \bigg).
\end{equation}

\begin{Lemma}
\label{BSSBM-N-Lemma 2}
Consider a sequence $f(n)$, and for each $n$, let $S_{1} \sim \text{Binom}  (\frac{n}{2}-1,p )$  and $S_{2} \sim \text{Binom}  (\frac{n}{2},q )$, where $p=a\frac{\log n}{n}$, and $q=b\frac{\log n}{n}$ for some $a\geq b >0$. Define $\omega \triangleq \lim_{n\rightarrow\infty} \frac{f(n)}{\log n}$. 
For sufficiently large $n$, when $\omega<\frac{a-b}{2}$,
\begin{equation*}
\mathbb{P} \big(S_{1}-S_{2} \leq f(n) \big) \leq n^{-\eta^{*}+o(1)} ,
\end{equation*}
where $\eta^{*} = \frac{a+b}{2}-\gamma^{*} -\frac{\omega}{2} \log \big ( \frac{a}{b}\big) +\frac{\omega}{2} \log  \Big( \frac{\gamma^{*} + \omega}{\gamma^{*} - \omega}  \Big)$ and $\gamma^{*} = \sqrt{\omega^{2}+ab}$.
\end{Lemma}
\begin{proof}
It follows from Chernoff bound.
\end{proof}

It follows from~\eqref{BSBM-N-equ101} that
\begin{align*}
\mathbb{P} ( d_{i}^{*}\leq \delta  ) =& \mathbb{P}  \Bigg( \sum_{j=1}^{n} G_{ij}x_{i}^{*}x_{j}^{*} \leq \frac{\delta-T_{2}}{T_{1}}   \Bigg)  (1-\alpha  ) \\
&+ \mathbb{P}  \Bigg( \sum_{j=1}^{n} G_{ij}x_{i}^{*}x_{j}^{*} \leq \frac{\delta+T_{2}}{T_{1}}   \Bigg) \alpha ,
\end{align*}
where $\sum_{j=1}^{n} G_{ij}x_{i}^{*}x_{j}^{*}$ is equal in distribution to $S_{1}-S_{2}$ in Lemma~\ref{BSSBM-N-Lemma 2}. 

Recall that $\beta \triangleq \lim_{n \rightarrow \infty} \frac{T_{2}}{\log n}$, where $\beta \geq 0$. First, we bound $\min_{i \in [n]}d_{i}^{*}$ under the condition $0 \leq \beta < \frac{T_{1}}{2}(a-b)$. It follows from Lemma~\ref{BSSBM-N-Lemma 2} that 
\begin{align*}
&\mathbb{P}  \Bigg( \sum_{j=1}^{n} G_{ij}x_{i}^{*}x_{j}^{*} \leq \frac{\delta-T_{2}}{T_{1}}   \Bigg) \leq n^{-\eta(a,b,\beta)+o(1)} , \\
&\mathbb{P}  \Bigg( \sum_{j=1}^{n} G_{ij}x_{i}^{*}x_{j}^{*} \leq \frac{\delta+T_{2}}{T_{1}}   \Bigg) \leq n^{-\eta(a,b,\beta)+\beta+o(1)} .
\end{align*}

Then
\begin{equation*}
\mathbb{P} ( d_{i}^{*} \leq \delta  ) \leq n^{-\eta(a,b,\beta) +o(1)} .
\end{equation*}
Using the union bound, 
\begin{equation*}
\mathbb{P} \bigg( \min_{i \in [n]}d_{i}^{*} \geq \frac{\log n}{\log \log n}   \bigg) \geq 1-n^{1-\eta(a,b,\beta)+o(1)} .
\end{equation*}
When $\eta(a,b,\beta)>1 $, it follows that $\min_{i \in [n]}d_{i}^{*} \geq \frac{\log n}{\log \log n}$ holds with probability $1-o(1)$. Substituting into~\eqref{BSSBM-N-equ3}, if $\eta(a,b,\beta)>1 $, then with probability $1-o(1)$,
\begin{equation*}
V^{T}S^{*}V \geq   ( 1-v^{2}  )  \bigg( \frac{\log n}{\log \log n} - T_{1} c' \sqrt{\log n}+T_{1}p   \bigg) > 0 , 
\end{equation*}
which concludes the first part of Theorem~\ref{Theorem 9}.

We now bound $\min_{i \in [n]}d_{i}^{*}$ under the condition $\beta>\frac{T_{1}}{2}(a-b)$. It follows from Lemma~\ref{BSSBM-N-Lemma 2} that 
\begin{align*}
\mathbb{P}  \bigg( \sum_{j=1}^{n} G_{ij}x_{i}^{*}x_{j}^{*} &\leq \frac{\delta-T_{2}}{T_{1}}  \bigg ) \leq n^{-\eta(a,b,\beta)+o(1)} ,\\
\mathbb{P}  \bigg( \sum_{j=1}^{n} G_{ij}x_{i}^{*}x_{j}^{*} &\leq \frac{\delta+T_{2}}{T_{1}}   \bigg) \leq 1 .
\end{align*}
Then
\begin{align*}
\mathbb{P} ( d_{i}^{*} \leq \delta  ) &\leq n^{-\eta(a,b,\beta)+o(1)} (1-\alpha ) +\alpha \\
& =n^{-\eta(a,b,\beta)+o(1)}+n^{-\beta+o(1)} ,
\end{align*}
where $\alpha =  n^{-\beta +o(1)}$. 
Using the union bound,
\begin{equation*}
\mathbb{P} \bigg( \min_{i \in [n]}d_{i}^{*} \geq \delta   \bigg) \geq 1 -n^{1-\eta(a,b,\beta)+o(1)} -n^{1-\beta+o(1)}.
\end{equation*}

\begin{Lemma}\cite[Lemma 8]{Ref7}
\label{BSSBM-N-Lemma 3}
When $\beta > 1$, then $\eta(a,b,\beta) > 1$.
\end{Lemma}

When $\beta>1 $, using Lemma~\ref{BSSBM-N-Lemma 3}, it follows that $\min_{i \in [n]}d_{i}^{*} \geq \frac{\log n}{\log \log n}$ holds with probability $1-o(1)$. Substituting into~\eqref{BSSBM-N-equ3}, if $\beta>1 $, then with probability $1-o(1)$, 
\begin{equation*}
V^{T}S^{*}V \geq    ( 1-v^{2}  )   \bigg( \frac{\log n}{\log \log n} - T_{1} c' \sqrt{\log n}+T_{1}p   \bigg) > 0 ,
\end{equation*}
which concludes the second part of Theorem~\ref{Theorem 9}.

%%%%%%%%%%%%%%%%%%%%%%%%%%%%%%%%%%%%%%%%%%%%%%%%%%%%%%%%%%%%%%%%%%%%%%%%%%%%%%%%%%%%%%%%%%%%%%%%%%%%%%%%%%%%%%%%%%%%
\section{Proof of Theorem~\ref{Theorem 11}}
\label{Proof-Theorem-11}
We begin by deriving sufficient conditions for the solution of \ac{SDP}~\eqref{BSSBM-N-equ1} to match the true labels.
\begin{Lemma}
\label{BSSBM-G-Lemma 1}
The sufficient conditions of Lemma~\ref{BSSBM-N-Lemma 1} apply to the general side information \ac{SDP}~\eqref{BSSBM-G-equ1} by replacing $S_{A}^{*} = \tilde{Y}^{T}X^{*}$ and $S_B^*=-\tilde{Y}$.
\end{Lemma}

\begin{proof}
The proof is similar to the proof of Lemma~\ref{BSSBM-N-Lemma 1}.
\end{proof}

It suffices to show that $S^{*}$ defined by $S_{A}^{*}$, $S_{B}^{*}$, and $S_{C}^{*}$ satisfies other conditions in Lemma~\ref{BSSBM-G-Lemma 1} with probability $1-o(1)$. 
Let 
\begin{equation}
\label{BSBM-G-dstar}
d_{i}^{*}=T_{1} \sum_{j=1}^{n} G_{ij}x_{j}^{*}x_{i}^{*} + \tilde{y}_{i}x_{i}^{*} .
\end{equation}
Then $D^{*}X^{*} = T_{1}GX^{*}+\tilde{Y}$ and based on the definitions of $S_{A}^{*}$, $S_{B}^{*}$, and $S_{C}^{*}$ in  Lemma~\ref{BSSBM-G-Lemma 1}, $S^{*}$ satisfies the condition $S^{*} [1, X^{*T}]^T =0$.
It remains to show that~\eqref{BCBM-N-equ1 New} holds, i.e., $S^{*}\succeq 0$ and $\lambda_{2}(S^{*})>0$ with probability $1-o(1)$.
For any $V$ such that $V^{T}[1, X^{*T}]^T=0$ and $  \| V   \|=1$, we have
\begin{align*}
V^{T}S^{*}V=&v^{2} S_{A}^{*} -2vT_{2}U^{T}Y +U^{T}D^{*}U - T_{1}U^{T}GU \\
\overset{(a)}{\geq}&   ( 1-v^{2}   )  \bigg[\min_{i \in [n]}d_{i}^{*} - T_{1}  \| G-\mathbb{E}[G]   \| + T_{1} p \bigg] \\
&+v^{2}  \bigg[ \tilde{Y}^{T}X^{*} -2y_{max} \frac{\sqrt{n(1-v^{2})}}{|v|} -T_{1}\frac{p-q}{2}    \bigg] \\
\overset{(b)}{=}&    ( 1-v^{2}   )  \bigg[\min_{i\in[n]}d_{i}^{*} - T_{1}  \| G-\mathbb{E}[G]   \| + T_{1} p \bigg] ,
\end{align*}
where $(a)$ holds in a manner similar to \eqref{BCBM-N-equ4} and \eqref{BSBM-N-equ102}, and $(b)$ holds by applying Lemma~\ref{BCBM-G-New Lemma}. Then using Lemma~\ref{BSSBM-P-Lemma 2}, 
\begin{equation}
\label{BSSBM-G-equ3}
V^{T}S^{*}V \geq  ( 1-v^{2}   )   \bigg( \min_{i \in [n]}d_{i}^{*} - T_{1} c^{'}\sqrt{\log n} +T_{1}p   \bigg).
\end{equation}

It can be shown that $\sum_{j=1}^{n} G_{ij}x_{i}^{*}x_{j}^{*}$ in~\eqref{BSBM-G-dstar} is equal in distribution to $S_{1}-S_{2}$ in Lemma~\ref{BSSBM-N-Lemma 2}. Then 
\begin{equation*}
% \label{BSSBM-G-equ4}
\mathbb{P} ( d_{i}^{*}\leq \delta  ) = \sum_{m_{1}=1}^{M_{1}} \sum_{m_{2}=1}^{M_{2}} ... \sum_{m_{K}=1}^{M_{K}} P (m_{1}, ..., m_{K} ) , 
\end{equation*}
where 
\begin{align*}
P& (m_{1}, ..., m_{K} ) \triangleq \mathbb{P}  ( x_{i}^{*} =1  )e^{f_{2}(n)} \mathbb{P}   \bigg( S_{1}-S_{2} \leq \frac{\delta-f_{1}(n)}{T_{1}}   \bigg) \\
&+\mathbb{P}  ( x_{i}^{*} =-1  ) e^{f_{3}(n)} \mathbb{P}   \bigg( S_{1}-S_{2} \leq \frac{\delta+f_{1}(n)}{T_{1}}   \bigg).
\end{align*}

First, we bound $\min_{i \in [n]}d_{i}^{*}$ under the condition $|\beta_{1}| \leq \frac{T_{1}}{2} (a-b )$. It follows from Lemma~\ref{BSSBM-N-Lemma 2} that 
\begin{align*}
&\mathbb{P}  \bigg( S_{1}-S_{2} \leq \frac{\delta-f_{1}(n)}{T_{1}}   \bigg) \leq n^{-\eta(a,b,\beta_{1})+o(1)} , \\
& \mathbb{P}  \bigg( S_{1}-S_{2} \leq \frac{\delta+f_{1}(n)}{T_{1}}   \bigg) \leq n^{-\eta(a,b,\beta_{1})+\beta_{1}+o(1)} .
\end{align*}
Notice that
\begin{align*}
\beta \triangleq 
\lim_{n\rightarrow\infty} -\frac{\max (f_2(n),f_3(n))}{\log n} .
\end{align*}
When $\beta_{1} \geq 0$, $\lim_{n\rightarrow \infty} \frac{f_{2}(n)}{\log n}= -\beta$ and $\lim_{n\rightarrow \infty} \frac{f_{3}(n)}{\log n} = -\beta_{1} -\beta$. Then
\begin{equation*}
\mathbb{P} ( d_{i}^{*} \leq \delta  ) \leq n^{-\eta(a,b,\beta_{1})-\beta+o(1)}. 
\end{equation*}
When $\beta_{1} < 0$, $\lim_{n\rightarrow \infty} \frac{f_{3}(n)}{\log n} = -\beta$ and $\lim_{n\rightarrow \infty} \frac{f_{2}(n)}{\log n} =\beta_{1} -\beta $. Then
\begin{equation*}
\mathbb{P} ( d_{i}^{*} \leq \delta  ) \leq n^{-\eta(a,b,\beta_{1})+\beta_{1}-\beta+o(1)}=n^{-\eta(a,b,  | \beta_{1}   |)-\beta+o(1)}. 
\end{equation*}
Using the union bound,
\begin{equation*}
\mathbb{P} \bigg( \min_{i \in [n]}d_{i}^{*} \geq \frac{\log n}{\log \log n}   \bigg) \geq 1-n^{1-\eta(a,b,  | \beta_{1}   |) -\beta+o(1)} .
\end{equation*}
When $\eta(a,b,  | \beta_{1}   |) +\beta >1 $, it follows that $\min_{i \in [n]}d_{i}^{*} \geq \frac{\log n}{\log \log n}$ holds with probability $1-o(1)$. Substituting into~\eqref{BSSBM-G-equ3}, if $\eta(a,b,  | \beta_{1}   |) +\beta >1 $, then with probability $1-o(1)$,
\begin{equation*}
V^{T}S^{*}V \geq    ( 1-v^{2}  )   \bigg( \frac{\log n}{\log \log n} - T_{1} c' \sqrt{\log n}+T_{1}p    \bigg) > 0 , 
\end{equation*}
which concludes the first part of Theorem~\ref{Theorem 11}.

We now bound $\min_{i \in [n]}d_{i}^{*}$ under the condition $  | \beta_{1}   | \geq \frac{T_{1}}{2} (a-b )$. When $\beta_{1} >  0$, $\lim_{n\rightarrow \infty} \frac{f_{2}(n)}{\log n} = -\beta$ and $\lim_{n\rightarrow \infty} \frac{f_{3}(n)}{\log n} = -\beta_{1} -\beta$. Then
\begin{equation*}
\mathbb{P} ( d_{i}^{*} \leq \delta  ) \leq n^{-\beta+o(1)} + n^{-\beta-\beta_{1}+o(1)}. 
\end{equation*}
When $\beta_{1} <  0$, $\lim_{n\rightarrow \infty} \frac{f_{3}(n)}{\log n} = -\beta$ and $\lim_{n\rightarrow \infty} \frac{f_{2}(n)}{\log n} =\beta_{1} -\beta$. Then
\begin{equation*}
\mathbb{P} ( d_{i}^{*} \leq \delta  ) \leq n^{-\beta+\beta_{1}+o(1)} +n^{-\beta+o(1)} .
\end{equation*}
Using the union bound,
\begin{equation*}
\mathbb{P} \bigg( \min_{i \in [n]}d_{i}^{*} \geq \frac{\log n}{\log \log n}   \bigg) \geq 1-n^{1-  | \beta_{1}   | -\beta+o(1)} .
\end{equation*}
When $  | \beta_{1}   | +\beta >1 $, it follows that $\min_{i \in [n]}d_{i}^{*} \geq \frac{\log n}{\log \log n}$ holds with probability $1-o(1)$. Substituting into~\eqref{BSSBM-G-equ3}, if $  | \beta_{1}   | +\beta >1 $, then with probability $1-o(1)$,
\begin{equation*}
V^{T}S^{*}V \geq    ( 1-v^{2}  )  \bigg( \frac{\log n}{\log \log n} - T_{1} c' \sqrt{\log n}+T_{1}p    \bigg) > 0 , 
\end{equation*}
which concludes the second part of Theorem~\ref{Theorem 11}.

%%%%%%%%%%%%%%%%%%%%%%%%%%%%%%%%%%%%%%%%%%%%%%%%%%%%%%%%%%%%%%%%%%%%%%%%%%%%%%%%
%
%References are important to the reader; therefore, each citation must be complete and correct. If at all possible, references should be commonly available publications.

%\ifCLASSOPTIONcaptionsoff
%\newpage
%\fi 
\bibliographystyle{IEEEtran}
\bibliography{Ref}

% \vskip 2in

\end{document}